\documentclass[11pt]{article}

\usepackage[preprint]{acl}

\usepackage{times}
\usepackage{latexsym}

\usepackage[T1]{fontenc}

\usepackage[utf8]{inputenc}

\usepackage{microtype}

\usepackage{inconsolata}

\usepackage{graphicx}

\usepackage{booktabs}
\usepackage{amsmath}
\usepackage{amsthm}
\usepackage[table]{xcolor}
\usepackage{array}
\definecolor{tableheader}{RGB}{70, 90, 120}
\definecolor{tablesection}{RGB}{230, 236, 245}

\usepackage{CJKutf8}

\usepackage{amssymb}
\usepackage{pifont}
\newcommand{\cmark}{\ding{51}}%
\newcommand{\xmark}{\ding{55}}%

\usepackage{multirow} 

\usepackage{algorithm}
\usepackage{algpseudocode}

\usepackage{booktabs}
\usepackage{multirow}
\usepackage{tabularx}
\usepackage{array}

\newcolumntype{Y}{>{\centering\arraybackslash}X}
\usepackage{xcolor}
\usepackage{graphicx}

\usepackage{pifont}

\usepackage{amsthm}
\newtheorem{proposition}{Proposition}

\title{CORTEX: High-Quality Cross-Domain Organization of Web-Scale Corpora through Ontological Corpus Graph 
}

\author{
  \textbf{Chengtao Gan}\textsuperscript{$\spadesuit$},
  \textbf{Xiaoke Guo}\textsuperscript{$\spadesuit$},
  \textbf{Yushan Zhu}\textsuperscript{$\clubsuit$},
  \textbf{Zhaoyan Gong}\textsuperscript{$\spadesuit$},
\\
  \textbf{Zhiqiang Liu}\textsuperscript{$\spadesuit$},
  \textbf{Songze Li}\textsuperscript{$\spadesuit$},
  \textbf{Huajun Chen}\textsuperscript{$\spadesuit$},
  \textbf{Wen Zhang}\textsuperscript{$\spadesuit$\textdagger}
  \\
  $^{\spadesuit}$Zhejiang University \\
  $^{\clubsuit}$JIUTIAN Research, Beijing, China \\
  \texttt{\{chengtaogan,zhang.wen\}@zju.edu.cn}
}

\begin{document}
\maketitle
\begin{abstract}
The continuous evolution of large language models drives escalating demands on data scale and quality, and as different training stages impose increasingly tailored data requirements, systematic organization of high-quality corpora becomes indispensable. Existing corpus construction pipelines confine the resulting corpora to flat, undifferentiated document collections, universally lacking systematic knowledge organization. We present \textbf{\textsc{Cortex}}, to our knowledge the first framework that elevates web-scale corpus construction from flat document filtering to structured knowledge organization through an \textbf{Ontological Corpus Graph (OCG)}, a three-layer heterogeneous structure unifying a quality-refined content layer, a hierarchical lightweight ontology layer via LLM-driven automated evolution, and a cross-domain alignment layer enabling inter-domain association at arbitrary taxonomic resolution. Comprehensive experiments confirm the effectiveness of \textsc{Cortex}. In particular, we leverage the OCG to synthesize \textsc{CortexBench}, a cross-domain search-and-reasoning benchmark whose evaluation across eight frontier LLMs validates the effectiveness of quality refinement, domain organization, and cross-domain data synthesis. We will publicly release the complete codebase, a 24.14\,B-token refined corpus with its OCG, and \textsc{CortexBench}.
\end{abstract}

\section{Introduction}
\label{sec:introduction}

\begingroup
  \renewcommand{\thefootnote}{\fnsymbol{footnote}}
  \setcounter{footnote}{0}
  \footnotetext{\textsuperscript{\textdagger} Corresponding author}
\endgroup

Large language model (LLM) performance exhibits empirical power-law
scaling with both model size and training data
volume~\citep{scallinglaw,TrainCompute}, driving an ever-growing demand
for massive training corpora. Web corpora such as Common
Crawl~\citep{commoncrawl} have become the dominant source for
trillion-token-scale pretraining data; however, raw web text is
inherently noisy, containing substantial low-quality, redundant, or
harmful content. Recent studies further demonstrate that data quality
critically governs scaling efficiency, with noisy data degrading the
effective contribution of each additional training
token~\citep{FineWeb,DCLM}. The central challenge has thus shifted
from merely accumulating data to \emph{systematically distilling
high-quality content} from massive-scale noisy web sources.

Existing corpus construction pipelines have evolved from
rule-based heuristic filters~\citep{C4} to classifier-based quality
scoring~\citep{FineWeb,DCLM}; however, these methods typically
evaluate along a single quality dimension, whereas text quality
inherently demands multi-dimensional assessment. While frontier LLMs
can provide such assessment comparable to human
annotators~\citep{ChatGPTOutperforms}, deploying them at
trillion-token scale remains prohibitively expensive. Beyond quality,
existing pipelines universally lack \emph{systematic knowledge
organization}, uniformly producing flat document collections without
domain taxonomy or inter-domain relationship modeling
(Table~\ref{tab:comparison}), precluding fine-grained domain-specific
extraction and cross-domain association exploitation. These gaps
underscore the pressing need for frameworks that jointly address
quality refinement and knowledge organization. High-quality curated
corpora also remain highly imbalanced across
languages~\citep{LinguisticDiversity,CCI3HQ,ChineseWebText}, with
many widely spoken languages facing scarcity, further
motivating language-agnostic solutions.

To jointly address these challenges, we present
\textsc{\textbf{Cortex}} (\textbf{C}urated \textbf{O}ntological
\textbf{R}efinement and \textbf{T}iered Graph-\textbf{E}nabled
Cross-domain Ne\textbf{X}us), a framework applicable to any target
language that refines large-scale raw corpora and organizes the
resulting high-quality content with an
\textbf{\emph{Ontological Corpus Graph} (OCG)}, a three-layer
heterogeneous structure unifying content refinement, lightweight
ontology construction, and cross-domain association modeling
(Figure~\ref{fig:framework}). \textbf{(1)~The \emph{High Quality
Content Layer}} is constructed through knowledge distillation-based
quality assessment; we design an
\textit{\emph{Ordinal-Aware Regression} (OAR)} loss to distill
teacher LLMs' multi-dimensional scoring capabilities into a
0.3\,B-parameter student with over $1{,}000\!\times$ compression,
enabling reliable quality assessment at trillion-token scale.
\textbf{(2)~The \emph{Lightweight Ontology Layer}} provides a
hierarchical lightweight ontology constructed through LLM-driven
automated evolution, and \textbf{(3)~the \emph{Alignment Layer}}
bridges content and the Lightweight Ontology via
Typicality--Fidelity--Specificity (TFS) weighted linkages, enabling
inter-domain association analysis and domain-aware retrieval at
arbitrary taxonomic resolution. 

We validate \textsc{Cortex} through three complementary experiments:
(i)~knowledge distillation pipeline validation with corpus and OCG
distributional analysis, (ii)~continual pre-training across four seed
models on OCG-organized domain-specific and neighbor-chain data, and
(iii)~OCG-driven synthesis of \textsc{CortexBench}, a cross-domain
search-and-reasoning QA benchmark evaluated on eight frontier LLMs
under a stable protocol. Together, these experiments
validate the effectiveness of quality refinement, domain
organization, inter-domain association modeling, and cross-domain
data synthesis.

Our main contributions are as follows:
\begin{itemize}
\item We propose \textsc{\textbf{Cortex}}, to our knowledge the first
  web-scale corpus construction framework to unify content quality
  refinement, lightweight ontology construction, and cross-domain
  association modeling via an \textbf{\emph{Ontological Corpus Graph}
  (OCG)}.
\item We validate \textsc{Cortex} through continual pre-training
  across four seed models and further design an OCG-driven data synthesis method and apply it to construct
  \textsc{CortexBench}; evaluation on eight frontier LLMs confirms the
  high difficulty and discriminative power of the synthesized
  instances, jointly validating the corpus quality, OCG-based data
  organization, and synthesis methodology.
\item We will publicly release the complete \textsc{Cortex} codebase,
  the trained scoring models, a 24.14\,B-token refined 
corpus with its accompanying OCG, and
  \textsc{CortexBench}.
\end{itemize}

\section{Related Work}
\label{sec:RelatedWork}

\paragraph{Corpus Quality}
Early pipelines such as C4~\citep{C4} rely on rule-based heuristic
filters with quality signals confined to syntactic
properties~\citep{FinerWeb-10BT}. Subsequent work introduces
model-based quality classification~\citep{FineWeb,DCLM}, and
Nemotron-CC~\citep{Nemotron-CC} deploys classifiers for quality,
topic, and synthetic content detection.
QuRating~\citep{QuRating} collects pairwise LLM judgments across
four quality criteria.
While frontier LLMs have demonstrated assessment quality comparable
to human annotators~\citep{ChatGPTOutperforms}, deploying them
directly at trillion-token scale remains prohibitively expensive.
Knowledge distillation~\citep{Distilling} offers a pathway to
compress such capabilities into lightweight models.

\paragraph{Knowledge Organization}
Existing large-scale corpus construction methods focus predominantly
on filtering mechanism
design~\citep{FineWeb,DCLM,RefinedWeb,Nemotron-CC}, producing flat
document collections without knowledge organization
(Table~\ref{tab:comparison}). Automated taxonomy construction has
evolved from corpus-driven term
clustering~\citep{TaxoGen} to LLM-based methods such as
Chain-of-Layer~\citep{Chain-of-Layer} and
TaxoAdapt~\citep{TaxoAdapt}, but these remain limited to
domain-specific settings.

\paragraph{Cross-Domain Synthesis}
LLM-based data synthesis is effective for training and evaluation
resource construction~\citep{SelfInstruct,Phi}, but quality
critically depends on source
documents~\citep{Source2Synth}. Existing methods primarily operate
on individual documents without exploiting cross-domain
relationships. Cross-document reasoning benchmarks such as
HotpotQA~\citep{HotpotQA} and MuSiQue~\citep{MuSiQue} advance
multi-passage reasoning but lack cross-domain customization, and web
search evaluation~\citep{BrowseComp} lacks the stability required
for controlled assessment~\citep{BrowseComp-Plus}.

\section{Methodology}
\label{sec:methodology}

We present \textsc{Cortex} (\textbf{C}urated \textbf{O}ntological \textbf{R}efinement and \textbf{T}iered Graph-\textbf{E}nabled Cross-domain Ne\textbf{X}us), 
a framework that distills
high-quality corpora from web-scale corpora and organizes them
with an \textbf{\emph{Ontological Corpus Graph} (OCG)}. As shown in
Figure~\ref{fig:framework}, the OCG is a three-layer heterogeneous
structure formally defined as
\begin{equation}
  \mathcal{G}
    = \bigl\langle\,\mathcal{L}_C,\;\mathcal{L}_O,\;
      \mathcal{L}_A\,\bigr\rangle.
  \label{eq:ocg}
\end{equation}

\begin{figure*}[t]
  \centering
  \includegraphics[width=\textwidth]{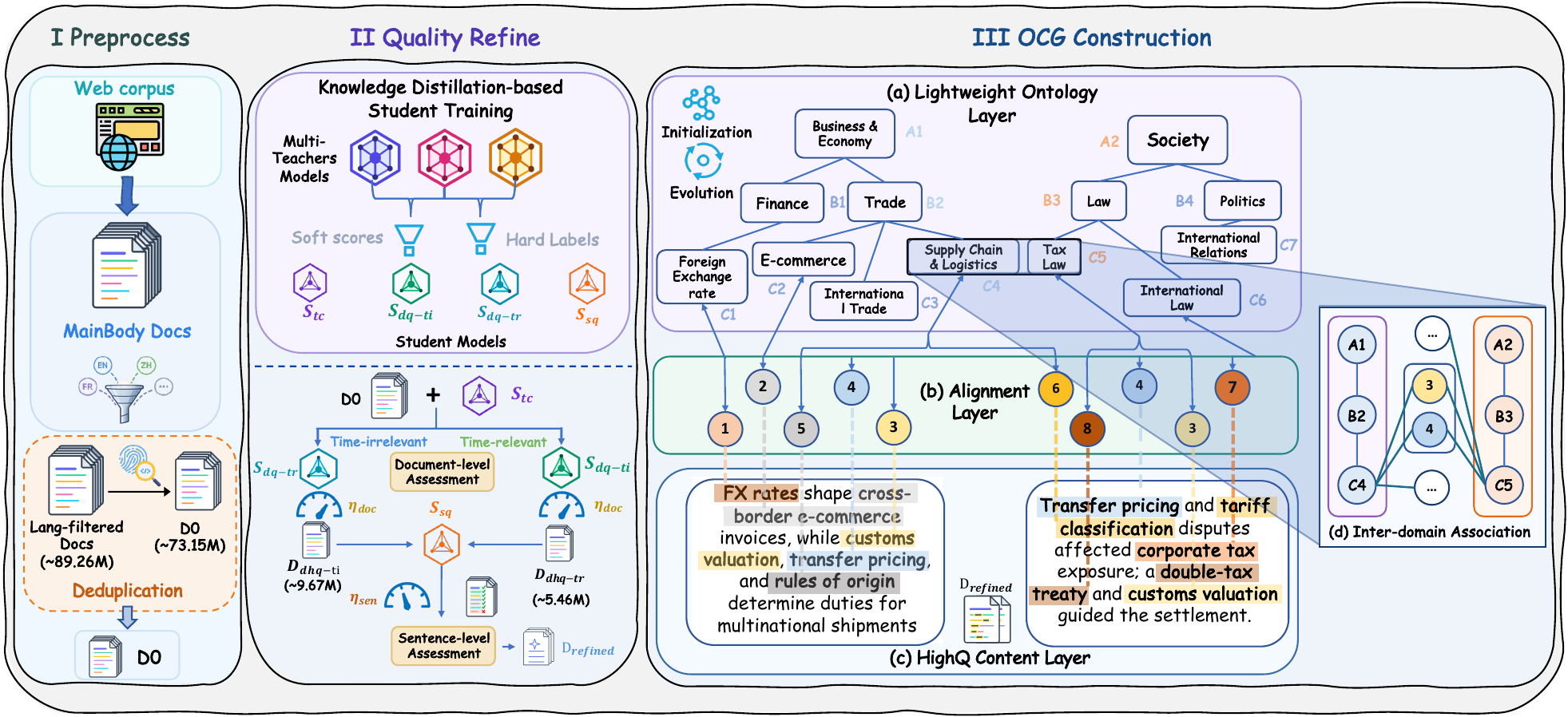}
  \caption{
    Overview of the \textsc{Cortex} framework.
  }
  \label{fig:framework}
\end{figure*}

The \textbf{High Quality Content Layer}
$\mathcal{L}_C = \mathcal{D}_{\text{refined}}$
(Stage~III\,(c) in Figure~\ref{fig:framework};
\S\ref{sec:content_layer}) is a quality-refined document corpus
constructed through multi-stage preprocessing and knowledge
distillation-based quality assessment (Stages~I--II).
The \textbf{Lightweight Ontology Layer}
$\mathcal{L}_O = \mathcal{C}$
(Stage~III\,(a) in Figure~\ref{fig:framework};
\S\ref{sec:ontology_layer}) is a hierarchical Concept Chain Set
constructed via LLM-driven automated evolution.
The \textbf{Alignment Layer}
$\mathcal{L}_A = (\mathcal{K},\,\Phi,\,W)$
(Stage~III\,(b) in Figure~\ref{fig:framework};
\S\ref{sec:alignment_layer}) bridges $\mathcal{L}_C$ and
$\mathcal{L}_O$, enabling inter-domain association analysis at
arbitrary taxonomic resolution. Table~\ref{tab:comparison} compares
\textsc{Cortex} with representative pipelines.

\subsection{High Quality Content Layer}
\label{sec:content_layer}
\label{sec:kdqa}

The High Quality Content Layer $\mathcal{L}_C$ is constructed through
two stages. Stage~I (\S\ref{sec:preprocessing}) converts raw WARC
records from Common Crawl into a preprocessed
corpus~$\mathcal{D}_0$. Stage~II refines $\mathcal{D}_0$ into
$\mathcal{L}_C = \mathcal{D}_{\text{refined}}$ via knowledge
distillation-based quality assessment
(\S\ref{sec:teacher}--\ref{sec:sen_qa}).

\subsubsection{Data Acquisition and Preprocessing}
\label{sec:preprocessing}

As shown in Stage~I of Figure~\ref{fig:framework}, this stage
converts raw WARC records into the preprocessed corpus
$\mathcal{D}_0$ through three standard steps with two key
distinctions from prior pipelines. First, in addition to
Resiliparse~\cite{Resiliparse,resiliparse2} for high-throughput body
text extraction~\cite{DCLM}, we employ
Trafilatura~\citep{Trafilatura} for structured metadata extraction
(e.g., publication date), which is not retained by existing pipelines
but is leveraged by our temporal relevance classifier
(\S\ref{sec:teacher}). Second, our three-stage coarse-to-fine
language filtering incorporates the document-level quality scorer
(\S\ref{sec:doc_qa}) as an implicit final language purity check.
Global exact deduplication follows snapshot-wide
practices~\citep{FineWeb,Nemotron-CC}. Full preprocessing details
are in Appendices~\ref{app:preprocessing_details}--\ref{app:dedup}.

\begin{table}[t]
\centering
\small
\setlength{\tabcolsep}{2.5pt}
\renewcommand{\arraystretch}{1.08}
\begin{tabular}{@{}l@{\hspace{4pt}}ccccc@{}}
\toprule
\textbf{Procedure}
  & \textbf{\textsc{Cortex}}
  & \textbf{FineW}
  & \textbf{DCLM}
  & \textbf{C4}
  & \textbf{Ne-CC} \\
\midrule
\multicolumn{6}{@{}l}{\textbf{\textit{Preprocessing}}} \\
\quad HTML-to-Text        & \cmark & \cmark  & \cmark  & $\circ$ & \cmark \\
\quad Body extraction     & \cmark & \cmark  & \cmark  & \xmark  & \cmark \\
\quad Metadata extr.      & \cmark & \xmark  & \xmark  & \xmark  & \xmark \\
\quad Language ID         & \cmark & \cmark  & \cmark  & \cmark  & \cmark \\
\quad Deduplication       & \cmark & \cmark  & \cmark  & \cmark  & \cmark \\
\midrule
\multicolumn{6}{@{}l}{\textbf{\textit{Quality Assessment}}} \\
\quad Causal reasoning    & \cmark & \xmark  & \xmark  & \xmark  & \xmark  \\
\quad Q\&A / explan.      & \cmark & \xmark  & $\circ$ & \xmark  & $\circ$ \\
\quad Factual accuracy    & \cmark & $\circ$ & \xmark  & \xmark  & $\circ$ \\
\quad Expression qua.  & \cmark & $\circ$ & \xmark  & $\circ$ & $\circ$ \\
\quad Info. density       & \cmark & $\circ$ & \xmark  & \xmark  & $\circ$ \\
\quad Harmful content     & \cmark & $\circ$ & $\circ$ & \cmark  & $\circ$ \\
\quad Temporal  & \cmark & \xmark  & \xmark  & \xmark  & \xmark  \\
\midrule
\multicolumn{6}{@{}l}{\textbf{\textit{Cross-domain Modeling}}} \\
\quad Domain class.       & OCG    & \xmark  & \xmark  & \xmark  & \xmark \\
\quad Inter-domain ass. & OCG    & \xmark  & \xmark  & \xmark  & \xmark \\
\bottomrule
\multicolumn{6}{@{}p{0.95\linewidth}@{}}{\scriptsize
  \cmark\;explicitly addressed;\quad
  $\circ$\;partially captured;\quad
  \xmark\;not addressed.} \\
\end{tabular}
\caption{Comparison of \textsc{Cortex} with representative web corpus
pipelines.
Unabbreviated version with full row
descriptions is in Appendix~\ref{app:comparison_table}.}
\label{tab:comparison}
\end{table}

\subsubsection{Multi-Teacher Annotation}
\label{sec:teacher}

For each scoring task, a stratified sample
$\mathcal{D}_{\text{anno}} \subset \mathcal{D}_0$ is drawn and annotated
by an ensemble of frontier LLMs
$\mathcal{M}=\{\mathcal{M}_1,\ldots,\mathcal{M}_{|\mathcal{M}|}\}$
serving as teachers (specific model selection are detailed
in Appendix~\ref{app:teacher_models}). For each sample
$x_n \in \mathcal{D}_{\text{anno}}$, each teacher $\mathcal{M}_m$
independently produces a quality score, and the consensus soft label is
$y_n = \frac{1}{|\mathcal{M}|}\sum_{m=1}^{|\mathcal{M}|}
\mathcal{M}_m(x_n) \in [0,1]$.
This continuous score is then mapped to a hard label via fixed thresholds:
\textsc{LowQ} for $y_n\in[0,\,0.4]$, \textsc{MidQ} for
$y_n\in(0.4,\,0.7)$, and \textsc{HighQ} for $y_n\in[0.7,\,1.0]$.

In contrast to existing pipelines that assess quality along a single
axis~\citep{FineWeb}, our rubric integrates multiple quality
dimensions, including causal reasoning, factual accuracy, expression
quality, and harmful content filtering, into a unified All-in-One
assessment (Table~\ref{tab:comparison}; full rubric in
Appendix~\ref{app:prompts}). Teachers additionally produce a binary
\emph{temporal relevance} label ($\textsc{Time\text{-}Rel}$ or
$\textsc{Time\text{-}Irrel}$), indicating whether the content is
time-sensitive or time-invariant.

\subsubsection{Ordinal-Aware Regression (OAR) Loss}
\label{sec:oar_loss}

Training the student to predict quality scores $\hat{y}_n \in [0,1]$
requires simultaneously learning the teacher's continuous scoring
standard and its discrete classification boundaries, as purely
regression-based or classification-based training each sacrifices one
of these objectives. Inspired by ordinal regression
methods~\cite{CORAL,CORN}, we propose the
\textbf{Ordinal-Aware Regression (OAR)} loss, which jointly optimizes
a robust regression objective and a soft ordinal boundary objective:
\begin{equation}
  \mathcal{L}_{\text{OAR}}
    = \lambda_{\text{reg}}\,\mathcal{L}_{\text{reg}}
    + \lambda_{\text{ord}}\,\mathcal{L}_{\text{ord}}.
  \label{eq:total_loss}
\end{equation}

Both loss components incorporate adaptive sample weights $w_n$ that
emphasize boundary-adjacent and target-class samples; the full
weighting scheme is detailed in Appendix~\ref{app:sample_weighting}.

\paragraph{Weighted Huber Regression Loss.}
To robustly fit the teacher's continuous scores,
we employ a weighted Huber loss~\cite{huberloss}:
$\mathcal{L}_{\text{reg}}
    = \tfrac{1}{N}\sum_{n=1}^{N}w_n
      \operatorname{Huber}_\delta(\hat{y}_n - y_n)$,
where $\operatorname{Huber}_\delta(e)=\tfrac{1}{2}e^2$ if
$|e|\le\delta$ and $\delta(|e|-\tfrac{1}{2}\delta)$ otherwise.

\paragraph{Soft-Threshold Ordinal Boundary Loss.}
Let $t_1<t_2$ denote the two ordinal boundaries. We define a
temperature-scaled boundary sigmoid
$c_j(s;\,\tau') = \sigma\bigl((t_j - s)/\tau'\bigr)$ for
$j\in\{1,2\}$, where $\tau'{>}0$ controls boundary sharpness,
inducing a smooth three-class distribution with components
$q_{\text{L}}{=}c_1$,\;
$q_{\text{M}}{=}c_2{-}c_1$,\;
$q_{\text{H}}{=}1{-}c_2$.
We construct the \emph{predicted} distribution
$\hat{\mathbf{p}}_n =
[\hat{c}_1,\;\hat{c}_2{-}\hat{c}_1,\;1{-}\hat{c}_2]$ with
$\hat{c}_j = c_j(\hat{y}_n;\tau)$, and the \emph{soft target}
$\mathbf{p}_n^* = [c_1^*,\;c_2^*{-}c_1^*,\;1{-}c_2^*]$ with
$c_j^* = c_j(y_n;\tau^*)$. The ordinal boundary loss is:
\begin{equation}
  \mathcal{L}_{\text{ord}}
    = -\frac{1}{N}\!\sum_{n=1}^{N} w_n
      \!\!\sum_{c\in\{\text{L,M,H}\}}\!\!
      p_{n,c}^*\log\hat{p}_{n,c}.
  \label{eq:ord_loss}
\end{equation}
Further details are in Appendix~\ref{app:oar_temperature}.

\subsubsection{Knowledge Distillation-Based Student Training and
Assessment}
\label{sec:student_train}
\label{sec:doc_qa}
\label{sec:sen_qa}

We train lightweight student models $\mathcal{S}_\theta$ to replicate
the teacher ensemble's scoring behavior, producing
$\hat{y}_n = \mathcal{S}_\theta(x_n) \in [0,1]$
(Appendix~\ref{app:student_arch}).
Given the teacher-annotated set $\{(x_n,y_n)\}_{n=1}^{N}$, the
student is optimized by:
\begin{equation}
  \theta^* = \arg\min_\theta \frac{1}{N}\sum_{n=1}^{N}
    \ell\bigl(\mathcal{S}_\theta(x_n),\,y_n\bigr),
  \label{eq:training_obj}
\end{equation}
where $\ell = \mathcal{L}_{\text{OAR}}$ (Eq.~\ref{eq:total_loss}) for
quality scoring and $\ell = \mathcal{L}_{\text{BCE}}$ for temporal
classification. We train four models: a temporal classifier
$\mathcal{S}_{\text{tc}}^*$, document-level quality scorers
$\mathcal{S}_{\text{dq-tr}}^*$ (\textsc{Time\text{-}Rel}) and
$\mathcal{S}_{\text{dq-ti}}^*$ (\textsc{Time\text{-}Irrel}), and a
sentence-level scorer $\mathcal{S}_{\text{sq}}^*$. Post-hoc decision
thresholds $\eta_{\text{doc}}$ and $\eta_{\text{sen}}$ are calibrated
on held-out sets (Appendix~\ref{app:threshold}). Further details are
in Appendix~\ref{app:training}.

The trained models are applied sequentially to $\mathcal{D}_0$
(Stage~II in Figure~\ref{fig:framework}).
$\mathcal{S}_{\text{tc}}^*$ first partitions documents into
$\mathcal{D}_0^{\text{TR}}$ and $\mathcal{D}_0^{\text{TI}}$. Each
subset is scored by its dedicated quality scorer; documents exceeding
$\eta_{\text{doc}}$ form $\mathcal{D}_{\text{dhq}}$. Each
$d\in\mathcal{D}_{\text{dhq}}$ is segmented into sentences
$\{s_1,\ldots,s_{N_d}\}$
(Appendix~\ref{app:inference_details}); $\mathcal{S}_{\text{sq}}^*$
retains the subset
$\mathcal{S}_d^{+} = \{s_i
    \mid \hat{y}_{\text{sq}}(s_i)\ge\eta_{\text{sen}}\}$,
whose sentences are concatenated in original order to produce
$\tilde{d}$.
The refined corpus
$\mathcal{D}_{\text{refined}}
    = \{\tilde{d} \mid d \in \mathcal{D}_{\text{dhq}},\;
      \mathcal{S}_d^{+} \neq \emptyset\}$
completes $\mathcal{L}_C$.

\subsection{Lightweight Ontology Layer}
\label{sec:ontology_layer}
\label{sec:concept_chain}

The Lightweight Ontology Layer
$\mathcal{L}_O = \mathcal{C}$ provides a multi-resolution hierarchical
domain taxonomy for organizing the refined corpus
(Stage~III\,(a) in Figure~\ref{fig:framework}).
further details are in
Appendix~\ref{app:concept_hierarchy}.

\subsubsection{Initialization and Evolution}
\label{sec:concept_init}

\paragraph{Concept Chain Set Initialization.}
The initial Concept Chain Set $\mathcal{C}^{(0)}$ is bootstrapped from
Wikipedia's portal and manually reviewed (Appendix~\ref{app:concept_hierarchy}). It comprises 196
leaf-level concept chains spanning 12 top-level domains.

\paragraph{LLM-Driven Automated Evolution.}
To maximize the Concept Chain Set's coverage of the domains present in
$\mathcal{D}_{\text{refined}}$, we introduce an automated evolution
procedure. A small random subset
$\mathcal{D}_{\text{evo}} \subset \mathcal{D}_{\text{refined}}$ is
sampled to drive the evolution. Documents in $\mathcal{D}_{\text{evo}}$
are processed sequentially; at each step $t$, a frontier LLM
$\mathcal{M}_{\text{evo}}$ takes document $d_t$ and the current Concept
Chain Set $\mathcal{C}^{(t-1)}$ as input:
\begin{equation}
  \mathcal{M}_{\text{evo}}\bigl(d_t,\;\mathcal{C}^{(t-1)}\bigr)
    \;\to\;
    \bigl\{(c_i,\;\gamma_i)\bigr\}_{i=1}^{n_t},
  \label{eq:evolution}
\end{equation}
producing $n_t \in \{1,2,3\}$ chain--confidence pairs with
$\gamma_i \in [0,1]$. Two outcomes are possible:
\textbf{(a)}~when at least one existing chain achieves confidence
$\gamma_i \ge \tau_{\text{conf}}$, all returned chains satisfy
$c_i \in \mathcal{C}^{(t-1)}$ and the set remains unchanged:
$\mathcal{C}^{(t)} = \mathcal{C}^{(t-1)}$;
\textbf{(b)}~when no existing chain reaches $\tau_{\text{conf}}$, the LLM
proposes a candidate new chain
$c_{\text{new}} \notin \mathcal{C}^{(t-1)}$ subject to structural
constraints (Appendix~\ref{app:evolution_constraints}), which is
\emph{directly incorporated} into the set:
$\mathcal{C}^{(t)} = \mathcal{C}^{(t-1)} \cup \{c_{\text{new}}\}$.
Newly added chains immediately become available for subsequent matching,
enabling the taxonomy to grow incrementally throughout the evolution
process.

\subsubsection{Convergence Monitoring and Post-Evolution Admission
Control}
\label{sec:concept_conv}

Convergence is monitored via a sliding-window update rate
(Appendix~\ref{app:convergence}); evolution terminates when no
proposal occurs for $P$ consecutive documents. A post-hoc
\emph{admission control} step then filters candidate chains
$\mathcal{C}_{\text{cand}} =
  \mathcal{C}^{(T_{\text{stop}})} \setminus \mathcal{C}^{(0)}$,
where $T_{\text{stop}}$ is the termination step. Each candidate
$c \in \mathcal{C}_{\text{cand}}$ is scored as
$\operatorname{Score}(c)
    = \log_{10}(N_{\text{hit}}{+}1)
      \cdot(w_1\bar{C}+w_2 P_{\text{high}})$,
where $N_{\text{hit}}$ is the number of documents matched to $c$, $\bar{C}$ the average matching confidence,
$P_{\text{high}}$ the fraction with confidence
$>\tau_{\text{conf}}$, and $w_1,w_2$ are balancing weights. The
finalized Lightweight Ontology Layer is
$\mathcal{L}_O = \mathcal{C} = \mathcal{C}^{(0)} \cup
  \{c \in \mathcal{C}_{\text{cand}} \mid
    \operatorname{Score}(c) \ge \tau_{\text{admit}}\}$,
containing \textbf{403~leaf-level concept chains}
(Appendix~\ref{app:hyperparams}).

\subsection{Alignment Layer}
\label{sec:alignment_layer}
\label{sec:ocg}

The Alignment Layer
$\mathcal{L}_A = (\mathcal{K},\,\Phi,\,W)$ bridges $\mathcal{L}_C$ and
$\mathcal{L}_O$ by establishing keyword-mediated linkages between
documents and concept chains
(Stage~III\,(b) in Figure~\ref{fig:framework}).

\subsubsection{Document--Concept Association}
\label{sec:doc_concept}

Each document in $\mathcal{D}_{\text{refined}}$ is associated with the
most relevant concept chains from $\mathcal{C}$. Each chain $c$ is
first expanded into a natural-language description $\mathcal{E}(c)$
(Appendix~\ref{app:expansion_prompt}). Let $\mathbf{v}_d$ and
$\mathbf{v}_c$ denote the embeddings of document $d$ and
$\mathcal{E}(c)$, respectively. The association confidence is
$\alpha(d,c) = \cos(\mathbf{v}_d,\;\mathbf{v}_c)$, and we retain the
top-$M$ ($M{=}3$) chains for each document:
$\mathcal{C}_d = \operatorname{top\text{-}}M_{c\in\mathcal{C}}
\alpha(d,c)$. Further details are in
Appendix~\ref{app:hyperparams}.

\subsubsection{Keyword Extraction and Linking}
\label{sec:keyword}

For each document in $\mathcal{D}_{\text{refined}}$, we extract the
top-$T$ ($T{=}5$) keywords
(Appendix~\ref{app:keyword_details}). Each keyword $k$ in document
$d$ is then linked to its single most relevant concept chain among
$\mathcal{C}_d$:
$\phi_d(k) = \operatorname*{arg\,max}_{c\in\mathcal{C}_d}
    \cos(\mathbf{v}_k,\;\mathbf{v}_c)$,
where $\mathbf{v}_k$ is the embedding of keyword $k$. The linkage
functions $\Phi = \{\phi_d\}_{d\in\mathcal{L}_C}$ constitute the
second component of $\mathcal{L}_A$. The linking confidence is
$\operatorname{conf}_d(k)=\cos(\mathbf{v}_k,\,
\mathbf{v}_{\phi_d(k)})$.

\subsubsection{TFS Weighting and Inter-domain Association}
\label{sec:graph_construct}

The Alignment Layer encodes inter-domain relationships through
keyword-mediated connections between concept chains. Let
$\mathcal{D}:=\mathcal{D}_{\text{refined}}$ denote the document set,
$\mathcal{C}$~the Concept Chain Set ($|\mathcal{C}|{=}403$),
$\mathcal{K}$~the global keyword vocabulary, and
$\mathcal{C}_d$~the top-$M$ chains of document~$d$.

\paragraph{TFS Global Connection Weight.}
The TFS weight $W_{c,k}$ between a concept chain $c$ and a keyword $k$
integrates three factors, \emph{Typicality}, \emph{Fidelity}, and
\emph{Specificity}:
\begin{equation}
  W_{c,k}
    = P(k\!\mid\!c)
      \!\cdot\!\operatorname{AC}(c,k)
      \!\cdot\!\operatorname{IDF}_{\!\text{chain}}(k).
  \label{eq:global_weight}
\end{equation}
Let
$\mathcal{D}_c=\{d\!\in\!\mathcal{D}\mid c\!\in\!\mathcal{C}_d\}$ and
$\operatorname{F}(c,k)=
\sum_{d\in\mathcal{D}_c}\mathbb{I}[\phi_d(k){=}c]$, where the
indicator contributes $1$ only when $k$ is among the top-$T$ keywords of
$d$ and is linked to $c$. \emph{Typicality}
$P(k\!\mid\!c)=\operatorname{F}(c,k)/|\mathcal{D}_c|$ measures how
frequently $k$ is linked to $c$. \emph{Fidelity}
$\operatorname{AC}(c,k) {=}
{\sum\nolimits_{d\in\mathcal{D}_c}
            \operatorname{conf}_d(k)\,\mathbb{I}[\phi_d(k)=c]}\,/\,
           {\operatorname{F}(c,k)}$
captures the average linking confidence when $k$ is linked to $c$.
\emph{Specificity} penalizes keywords linked to many chains via a
\emph{chain-level} inverse document frequency:
$\operatorname{IDF}_{\!\text{chain}}(k)
  = \log\tfrac{|\mathcal{C}|+1}{|\mathcal{C}_k|}$,
where
$\mathcal{C}_k{=}\{c\!\in\!\mathcal{C}\mid
\operatorname{F}(c,k)>0\}$.

\paragraph{Inter-domain Association.}
As illustrated in Stage~III\,(d) of Figure~\ref{fig:framework}, the
OCG captures inter-domain relationships through association metrics
computable at \emph{any level} of the taxonomy. Each leaf chain $c$
has a TFS weight vector
$\mathbf{w}_c\in\mathbb{R}^{|\mathcal{K}|}$ with components
$[\mathbf{w}_c]_k = W_{c,k}$. For any hierarchical node $p$, let
$\mathcal{C}_p^{\downarrow}=\{c\in\mathcal{C}\mid
c\text{ starts with }p\}$ be its descendant leaf chains. We define the
accumulated signal vector
$\mathbf{s}_p =
\sum_{c\in\mathcal{C}_p^{\downarrow}}\mathbf{w}_c$, whose component
$S_{p,k}=\sum_{c\in\mathcal{C}_p^{\downarrow}}W_{c,k}$ aggregates
keyword weights across all descendant leaves. When $p$ is a leaf chain,
$\mathcal{C}_p^{\downarrow}{=}\{p\}$ and
$\mathbf{s}_p{=}\mathbf{w}_p$.
The association between any two nodes $p_i$ and $p_j$ is:
\begin{equation}
  \operatorname{Assoc}(p_i,p_j)
    = \frac{\sum\nolimits_{k\in\mathcal{K}}
            S_{p_i,k}\,S_{p_j,k}}
           {\|\mathbf{s}_{p_i}\|\;\|\mathbf{s}_{p_j}\|}.
  \label{eq:assoc}
\end{equation}
Further details are in Appendix~\ref{app:proofs}.

\paragraph{Concept Chain-to-Document Inverted Index.}
We construct an inverted index mapping each concept chain to its
associated documents, supporting efficient retrieval at any level of the
hierarchy. Together with the association metrics, the completed OCG
$\mathcal{G} =
\langle\mathcal{L}_C,\mathcal{L}_O,\mathcal{L}_A\rangle$
(Eq.~\ref{eq:ocg}) provides a unified, queryable structure for
domain-aware corpus organization, enabling flexible extraction of
high-quality data at arbitrary domain granularity and cross-domain scope.

\section{Experiments}
\label{sec:experiments}

We instantiate \textsc{Cortex} on Chinese, motivated by both the
substantial demand and the pronounced scarcity of high-quality Chinese
web corpora~\citep{CCI3HQ,ChineseWebText}, and evaluate along three
research questions:
\textbf{RQ1}: Can the knowledge distillation-based refinement
pipeline effectively distill high-quality content from web-scale corpora, and what distribution does the resulting corpus
and OCG exhibit?
\textbf{RQ2}: Is the refined corpus high-quality, and does
the OCG accurately organize and retrieve domain-specific content while
capturing latent inter-domain associations?
\textbf{RQ3}: Can the high-quality refined corpus and the constructed
OCG effectively support cross-domain synthesis of high-value data?

\subsection{Pipeline Validation and Analysis (RQ1)}
\label{sec:exp_rq1}

To address RQ1, we validate the effectiveness of the knowledge
distillation-based refinement pipeline and characterize the
distributional properties of the corpus and OCG. 
(Appendix~\ref{app:pipeline_stats}).

\begin{table}[t]
\centering\small
\setlength{\tabcolsep}{3pt}
\begin{tabular}{@{}llccc@{}}
\toprule
\textbf{Student} & \textbf{Task}
  & \textbf{Target} & \textbf{Recall}
  & $\boldsymbol{\rho}$ \\
\midrule
$\mathcal{S}_{\text{tc}}^*$
  & Temporal Classif.
  & 82.56\,$^{F_1}$
  & ---
  & --- \\
$\mathcal{S}_{\text{dq-tr}}^*$
  & Doc Quality (\textsc{TR})
  & 80.33\,$^{F_2}$
  & 98.00
  & .805 \\
$\mathcal{S}_{\text{dq-ti}}^*$
  & Doc Quality (\textsc{TI})
  & 84.25\,$^{F_2}$
  & 95.83
  & .922 \\
$\mathcal{S}_{\text{sq}}^*$
  & Sentence Quality
  & 96.30\,$^{F_2}$
  & 98.33
  & .908 \\
\bottomrule
\end{tabular}

\caption{Student model performance on held-out test sets.
Metric definitions are in Appendix~\ref{app:kd_detail}.}
\label{tab:kd_results}
\end{table}

\paragraph{Knowledge Distillation Effectiveness.}
Table~\ref{tab:kd_results} reports the performance of the four
knowledge-distilled student models on their held-out
natural-distribution test sets (Appendix~\ref{app:training}). All
students are initialized from MacBERT-Large (${\sim}$0.3\,B
parameters; Appendix~\ref{app:student_arch}) and trained to learn
the consensus scoring criteria jointly established by a
three-teacher LLM ensemble comprising
DeepSeek-V3.1~\cite{Deepseek-v3}, GPT-5~\citep{GPT-5}, and
Doubao-Seed-1.6. 
All quality scorers achieve target-class recall exceeding 95\%,
ensuring that virtually all teacher-endorsed high-quality content is
preserved during corpus-scale filtering.
The Spearman correlations of
0.805--0.922 confirm strong ordinal alignment with teacher consensus
soft scores, validating the OAR loss design (\S\ref{sec:oar_loss}).
These results demonstrate that the distillation pipeline compresses
the scoring behavior of a multi-hundred-billion-parameter teacher
ensemble into a single 0.3\,B student (${>}1{,}000{\times}$
compression; Appendix~\ref{app:model_selection}), enabling reliable
quality assessment at the scale of tens of billions of tokens.

\begin{table}[t]
\centering\small
\begin{tabular}{@{}lrrr@{}}
\toprule
\textbf{Tier} & \textbf{Documents} & \textbf{Fraction}
  & \textbf{Avg Tokens} \\
\midrule
\textsc{HighQ} & 15.13\,M & 20.69\% & 1{,}796 \\
\textsc{MidQ}  & 18.30\,M & 25.02\% & 1{,}288 \\
\textsc{LowQ}  & 39.72\,M & 54.30\% & 2{,}036 \\
\midrule
\textsc{TI}    & 36.50\,M & 49.91\% & 2{,}517 \\
\textsc{TR}    & 36.65\,M & 50.09\% & 1{,}085 \\
\midrule
Total          & 73.15\,M & 100\%   & 1{,}800 \\
\bottomrule
\end{tabular}
\caption{Quality-tier and temporal distribution of $\mathcal{D}_0$.
\textsc{TI}: \textsc{Time\text{-}Irrel};
\textsc{TR}: \textsc{Time\text{-}Rel}.
\textbf{Avg Tokens}: mean main-body token count per document.}
\label{tab:quality_dist}
\end{table}

\paragraph{Corpus Quality Distribution.}
Table~\ref{tab:quality_dist} summarizes the quality-tier and temporal
distributions of the preprocessed corpus $\mathcal{D}_0$ (73.15\,M
documents). \textsc{HighQ} documents account for 20.69\%,
constituting the 15.13\,M-document high-quality subset
$\mathcal{D}_{\text{dhq}}$ with 27.18\,B tokens. Sentence-level
filtering further refines this to $\mathcal{D}_{\text{refined}}$
with 24.14\,B tokens (Appendix~\ref{app:pipeline_stats}). 

While the temporal dimension is roughly balanced in
$\mathcal{D}_0$, the ratio shifts to 63.91\% \textsc{TI} in
$\mathcal{D}_{\text{dhq}}$ (Table~\ref{tab:dhq_temporal}), indicating
that time-invariant content (e.g., tutorials) is more likely to be of
higher quality. The domain distribution across 12 top-level domains
(Table~\ref{tab:domain_full}) reveals a pronounced long-tail pattern,
with five domains each contributing less than 2\% of high-quality
documents (Appendix~\ref{app:corpus_dist}).

\subsection{Corpus Quality and OCG Validation via Continual
  Pre-Training (RQ2)}
\label{sec:exp_rq2}

To address RQ2, we evaluate our method through continual pre-training (CPT) (Appendix~\ref{app:cpt_validation_design}).

\subsubsection{Experimental Setup}
\label{sec:cpt_setup}

\paragraph{Data Configurations.}
We select the \emph{Finance} domain
($c_{\text{fin}}=$ \texttt{Business \& Economy\,>\,Finance}) as a case
study and define two CPT data configurations using the OCG's inverted
index and association metric (Eq.~\ref{eq:assoc}). Let
$\mathcal{D}_c^{\text{TI}}=\{d\in\mathcal{D}_c \mid d\in
\mathcal{D}_0^{\text{TI}}\}$ denote the \textsc{Time\text{-}Irrel}
subset associated with chain~$c$, and
$\mathcal{N}_K(c) = \operatorname{top\text{-}}K_{c'\in\mathcal{C}
\setminus\{c\}}\operatorname{Assoc}(c,c')$ the $K$ nearest neighbor
chains:
\textbf{(i)}~\textbf{+\textsc{Fin}}:
$\mathcal{D}_{c_{\text{fin}}}^{\text{TI}}$, consisting of
time-invariant documents from the finance chain $c_{\text{fin}}$;
\textbf{(ii)}~\textbf{+\textsc{Fin}$^{++}$}:
$\mathcal{D}_{\text{fin}^{++}}^{\text{TI}} = \bigcup_{c'\in
\{c_{\text{fin}}\}\cup\mathcal{N}_{10}(c_{\text{fin}})}
\mathcal{D}_{c'}^{\text{TI}}$, additionally incorporating the top-10
neighbor chains of $c_{\text{fin}}$ retrieved via the OCG's
association metric (Table~\ref{tab:finance_neighbors}).

\paragraph{Training and Evaluation.}
We evaluate four seed models:
Qwen2.5-0.5B, Qwen2.5-3B, Qwen2.5-7B~\cite{QWen2.5},
and Llama-3-8B~\cite{LLama3-8B}.
All CPT is performed via full-parameter training(Table~\ref{tab:cpt_hparams}); 
\emph{Domain capability} is assessed on
CFBenchmark~\cite{CFBenchmark}, comprising six dimensions: Knowledge,
Calculation, Explanation, Identification, Analysis, and Compliance
(Appendix~\ref{app:cpt_detail}).

\subsubsection{Results}
\label{sec:cpt_results}

Table~\ref{tab:cpt_results} presents the results across all model and
data configurations.

\begin{table}[t]
\centering
\begingroup
\footnotesize
\setlength{\tabcolsep}{1.2pt}
\renewcommand{\arraystretch}{0.96}

\definecolor{CFGroupRow}{HTML}{F4F4F4}
\definecolor{CFBaseRow}{HTML}{FFFFFF}
\definecolor{CFFinRow}{HTML}{F8FBFE}
\definecolor{CFFinPPRow}{HTML}{F1F7FD}

\definecolor{CFAvgBase}{HTML}{F2F7FC}
\definecolor{CFAvgFin}{HTML}{E8F2FB}
\definecolor{CFAvgFinPP}{HTML}{DDEBFA}

\newcommand{\cfdatacfg}[1]{\hspace{0.35em}#1}
\newcommand{\cfbestavg}[1]{\textbf{#1}}
\newcommand{\cfsecondavg}[1]{\underline{#1}}

\newcommand{\avgbase}[1]{\cellcolor{CFAvgBase}#1}
\newcommand{\avgfin}[1]{\cellcolor{CFAvgFin}#1}
\newcommand{\avgfinpp}[1]{\cellcolor{CFAvgFinPP}#1}

\begin{tabularx}{\columnwidth}{@{}
  l
  *{7}{>{\centering\arraybackslash}X}
@{}}
\toprule
& \multicolumn{7}{c}{\textbf{CFBenchmark}} \\
\cmidrule(lr){2-8}
\textbf{Data}
  & \textbf{Know.}
  & \textbf{Calc.}
  & \textbf{Expl.}
  & \textbf{Ident.}
  & \textbf{Anal.}
  & \textbf{Compl.}
  & \textbf{Avg} \\
\midrule

\rowcolor{CFGroupRow}
\multicolumn{8}{@{}l}{\textbf{Qwen2.5-0.5B}} \\
\rowcolor{CFBaseRow}
\cfdatacfg{Base}
  & 14.0 & 15.9 & 12.7 & 50.7 & 19.5 & 4.0
  & \avgbase{19.5} \\
\rowcolor{CFFinRow}
\cfdatacfg{+\textsc{Fin}}
  & 14.0 & 23.2 & 28.3 & 46.7 & 19.8 & 4.7
  & \avgfin{\cfsecondavg{22.8}} \\
\rowcolor{CFFinPPRow}
\cfdatacfg{+\textsc{Fin}\textsuperscript{++}}
  & 17.3 & 24.1 & 30.9 & 54.7 & 20.1 & 4.0
  & \avgfinpp{\cfbestavg{25.2}} \\

\addlinespace[1pt]
\midrule
\rowcolor{CFGroupRow}
\multicolumn{8}{@{}l}{\textbf{Qwen2.5-3B}} \\
\rowcolor{CFBaseRow}
\cfdatacfg{Base}
  & 57.3 & 45.4 & 46.7 & 56.0 & 36.0 & 8.0
  & \avgbase{41.6} \\
\rowcolor{CFFinRow}
\cfdatacfg{+\textsc{Fin}}
  & 54.7 & 33.1 & 50.8 & 56.0 & 45.4 & 13.3
  & \avgfin{\cfsecondavg{42.2}} \\
\rowcolor{CFFinPPRow}
\cfdatacfg{+\textsc{Fin}\textsuperscript{++}}
  & 58.7 & 26.6 & 55.1 & 56.0 & 47.8 & 13.3
  & \avgfinpp{\cfbestavg{42.9}} \\

\addlinespace[1pt]
\midrule
\rowcolor{CFGroupRow}
\multicolumn{8}{@{}l}{\textbf{Qwen2.5-7B}} \\
\rowcolor{CFBaseRow}
\cfdatacfg{Base}
  & 38.0 & 48.9 & 46.7 & 42.7 & 41.4 & 18.7
  & \avgbase{39.4} \\
\rowcolor{CFFinRow}
\cfdatacfg{+\textsc{Fin}}
  & 76.0 & 47.9 & 54.9 & 53.3 & 43.4 & 29.3
  & \avgfin{\cfsecondavg{50.8}} \\
\rowcolor{CFFinPPRow}
\cfdatacfg{+\textsc{Fin}\textsuperscript{++}}
  & 86.7 & 49.8 & 51.2 & 57.3 & 46.2 & 21.3
  & \avgfinpp{\cfbestavg{52.1}} \\

\addlinespace[1pt]
\midrule
\rowcolor{CFGroupRow}
\multicolumn{8}{@{}l}{\textbf{Llama-3-8B}} \\
\rowcolor{CFBaseRow}
\cfdatacfg{Base}
  & 11.3 & 5.4 & 8.2 & 28.0 & 30.7 & 0.0
  & \avgbase{13.9} \\
\rowcolor{CFFinRow}
\cfdatacfg{+\textsc{Fin}}
  & 21.3 & 25.6 & 28.0 & 44.0 & 50.7 & 7.3
  & \avgfin{\cfsecondavg{29.5}} \\
\rowcolor{CFFinPPRow}
\cfdatacfg{+\textsc{Fin}\textsuperscript{++}}
  & 22.0 & 28.9 & 35.9 & 57.3 & 36.5 & 6.7
  & \avgfinpp{\cfbestavg{31.2}} \\

\bottomrule
\end{tabularx}

\caption{CPT results across all seed models and data configurations.
Full dimension names are in Appendix~\ref{app:cfb_benchmark}.}
\label{tab:cpt_results}
\endgroup
\end{table}

\paragraph{Domain Capability Enhancement.}
Comparing the Base, +\textsc{Fin}, and +\textsc{Fin}$^{++}$ rows
reveals consistent improvements across all four models; for instance,
Llama-3-8B improves from 13.9 to 29.5 in average score and
Qwen2.5-7B from 39.4 to 50.8, with +\textsc{Fin}$^{++}$ further
raising Qwen2.5-7B to 52.1. These gains jointly validate the quality
of the \textsc{Cortex}-curated corpus, the accuracy of the OCG's
domain partitioning and inverted-index retrieval, and the ability of
inter-domain associations to surface complementary cross-domain
training signal beyond the target domain alone. Such patterns hold
consistently across models spanning 0.5\,B to 8\,B parameters and
two model families (Appendix~\ref{app:cpt_detail}).

\begin{table*}[t]
\centering
\footnotesize
\setlength{\tabcolsep}{1.6pt}
\renewcommand{\arraystretch}{1.08}

\definecolor{PosDeltaColor}{HTML}{1B7F3A}
\definecolor{NegDeltaColor}{HTML}{B83232}

\newcommand{\scoreonly}[1]{%
  \makebox[2.7em][r]{#1}%
}
\newcommand{\poschg}[1]{%
  \makebox[2.85em][l]{%
    \hspace{0.16em}{\fontsize{6.1pt}{6.7pt}\selectfont
    \textcolor{PosDeltaColor}{\ensuremath{\uparrow}#1}}%
  }%
}
\newcommand{\negchg}[1]{%
  \makebox[2.85em][l]{%
    \hspace{0.16em}{\fontsize{6.1pt}{6.7pt}\selectfont
    \textcolor{NegDeltaColor}{\ensuremath{\downarrow}#1}}%
  }%
}
\newcommand{\scoredeltapos}[2]{%
  \makebox[2.7em][r]{#1}%
  \poschg{#2}%
}
\newcommand{\scoredeltaneg}[2]{%
  \makebox[2.7em][r]{#1}%
  \negchg{#2}%
}

\resizebox{\textwidth}{!}{%
\begin{tabular}{@{}l cccc cccc@{}}
\toprule
& \multicolumn{4}{c}{\textbf{Bridge}}
& \multicolumn{4}{c}{\textbf{Comparison}} \\
\cmidrule(lr){2-5}\cmidrule(lr){6-9}
\textbf{Model}
  & \textsc{Closed} & \textsc{1-Srch} & \textsc{M-Srch}
  & \textsc{Oracle}
  & \textsc{Closed} & \textsc{1-Srch} & \textsc{M-Srch}
  & \textsc{Oracle} \\
\midrule
DeepSeek-V3.2
  & \scoreonly{30.3}
  & \scoredeltapos{41.3}{11.0}
  & \scoredeltapos{52.2}{10.9}
  & \scoredeltapos{\underline{94.8}}{42.6}
  & \scoreonly{22.6}
  & \scoredeltapos{28.8}{6.2}
  & \scoredeltapos{47.2}{18.4}
  & \scoredeltapos{94.1}{46.9} \\
DeepSeek-V4-Flash
  & \scoreonly{30.9}
  & \scoredeltapos{44.0}{13.1}
  & \scoredeltapos{54.1}{10.1}
  & \scoredeltapos{90.2}{36.1}
  & \scoreonly{24.8}
  & \scoredeltapos{30.3}{5.5}
  & \scoredeltapos{49.1}{18.8}
  & \scoredeltapos{92.8}{43.7} \\
\midrule
Gemini 3 Pro
  & \scoreonly{30.9}
  & \scoredeltapos{38.1}{7.2}
  & \scoredeltapos{\textbf{63.0}}{24.9}
  & \scoredeltapos{80.1}{17.1}
  & \scoreonly{28.6}
  & \scoredeltapos{31.3}{2.7}
  & \scoredeltapos{50.9}{19.6}
  & \scoredeltapos{72.2}{21.3} \\
Gemini 3.1 Pro
  & \scoreonly{28.1}
  & \scoredeltapos{\textbf{65.7}}{37.6}
  & \scoredeltaneg{52.5}{13.2}
  & \scoredeltapos{82.5}{30.0}
  & \scoreonly{24.3}
  & \scoredeltapos{\textbf{54.1}}{29.8}
  & \scoredeltaneg{40.4}{13.7}
  & \scoredeltapos{65.6}{25.2} \\
\midrule
Claude Sonnet 4.6
  & \scoreonly{34.7}
  & \scoredeltapos{52.5}{17.8}
  & \scoredeltapos{59.1}{6.6}
  & \scoredeltapos{84.0}{24.9}
  & \scoreonly{25.0}
  & \scoredeltapos{41.9}{16.9}
  & \scoredeltaneg{41.0}{0.9}
  & \scoredeltapos{83.1}{42.1} \\
Claude Opus 4.6
  & \scoreonly{\underline{35.3}}
  & \scoredeltapos{50.9}{15.6}
  & \scoredeltapos{\underline{62.5}}{11.6}
  & \scoredeltapos{91.9}{29.4}
  & \scoreonly{\underline{31.7}}
  & \scoredeltapos{41.0}{9.3}
  & \scoredeltapos{\textbf{57.1}}{16.1}
  & \scoredeltapos{\textbf{95.7}}{38.6} \\
\midrule
GPT-4o
  & \scoreonly{26.6}
  & \scoredeltapos{26.9}{0.3}
  & \scoredeltapos{33.8}{6.9}
  & \scoredeltapos{77.5}{43.7}
  & \scoreonly{21.6}
  & \scoredeltapos{23.1}{1.5}
  & \scoredeltapos{28.1}{5.0}
  & \scoredeltapos{73.4}{45.3} \\
GPT-5.4
  & \scoreonly{\textbf{47.2}}
  & \scoredeltapos{\underline{58.4}}{11.2}
  & \scoredeltapos{60.9}{2.5}
  & \scoredeltapos{\textbf{96.1}}{35.2}
  & \scoreonly{\textbf{32.3}}
  & \scoredeltapos{\underline{48.4}}{16.1}
  & \scoredeltapos{\underline{51.6}}{3.2}
  & \scoredeltapos{\underline{94.5}}{42.9} \\
\bottomrule
\end{tabular}%
}
\caption{\textsc{CortexBench} results (Accuracy~\%).
\textsc{Closed}: closed-book.
\textsc{1-Srch}: single-round retrieval.
\textsc{M-Srch}: Multi-Search.
\textsc{Oracle}: gold evidence provided.
Small values indicate absolute changes from the previous setting.
}
\label{tab:benchmark_results}
\end{table*}

\subsection{Cross-Domain Synthesis Validation via
  \textsc{CortexBench} (RQ3)}
\label{sec:exp_rq3}

To address RQ3, we synthesize \textsc{CortexBench}, a cross-domain
search-and-reasoning QA benchmark, to validate the effectiveness of
combining the high-quality corpus with the OCG's structured
cross-domain organization for producing high-value cross-domain data.
We additionally release \textsc{CortexBench} to support future
research on cross-domain evaluation.

\subsubsection{Benchmark Construction}
\label{sec:bench_construct}

The synthesis pipeline produces QA pairs grounded in cross-domain
evidence and constructs per-instance candidate pools for controlled
evaluation. It proceeds in three stages
(Algorithm~\ref{alg:synthesis}).

\paragraph{OCG-Guided QA Synthesis.}
For each instance, a keyword $k\in\mathcal{K}$ is selected,
prioritizing named entities. Using
$\mathcal{C}_k{=}\{c{\in}\mathcal{C}\mid\operatorname{F}(c,k){>}0\}$
(\S\ref{sec:graph_construct}), we retrieve either a
\emph{high-association} chain pair
$(c_1^{+},c_2^{+}) =
\arg\max_{c_i\neq c_j\in\mathcal{C}_k}
\operatorname{Assoc}(c_i,c_j)$
or a \emph{low-association} pair (analogously via $\arg\min$).
For each pair $(c_1,c_2)$, the top-$N$ ($N{=}40$) documents per
chain are retrieved via the inverted index, and an LLM-based entity
extraction and matching procedure identifies an evidence pair
$(d_A,d_B)$ with $d_A\in\mathcal{D}_{c_1}$,
$d_B\in\mathcal{D}_{c_2}$. A frontier LLM then synthesizes bridge
and comparison QA pairs from $(d_A,d_B)$, each paired with a candidate pool $\mathcal{S}_{\text{cand}}$
($|\mathcal{S}_{\text{cand}}|{=}6{,}000$)
(Algorithm~\ref{alg:synthesis}; Appendix~\ref{app:synth_detail}).

\paragraph{Benchmark Statistics.}
\textsc{CortexBench} comprises 917 QA pairs organized along three
dimensions: question type, temporal source, and chain-pair association
level (Table~\ref{tab:bench_stats};
Appendix~\ref{app:bench_stats}).
The current instantiation utilizes
merely 42 of the over 1.3~million keywords in the OCG's vocabulary,
a utilization ratio below $1{:}30{,}000$, yet already yields 917
high-difficulty instances, demonstrating substantial scaling potential
for both evaluation and training data generation.
\label{sec:bench_results}

\subsubsection{Evaluation Protocol}
\label{sec:bench_eval}

For each instance $q$ with gold evidence $(d_A,d_B)$ and candidate
pool $\mathcal{S}_{\text{cand}}$
($|\mathcal{S}_{\text{cand}}|{=}6{,}000$), we define four
settings:
\textbf{(i)}~\textsc{Closed-book}: the model answers from parametric
knowledge alone (lower bound);
\textbf{(ii)}~\textsc{1-Search}: the model formulates a query and
retrieves chunks via
$\operatorname{Search}(\mathbf{q}_{\text{kw}},q_{\text{sum}},m)$, a
hybrid retrieval interface over $\mathcal{S}_{\text{cand}}$;
\textbf{(iii)}~\textsc{Multi-Search}: iterative query refinement via
ReAct~\citep{react};
\textbf{(iv)}~\textsc{Oracle}: gold evidence $(d_A,d_B)$ is provided
(upper bound).
Unlike live web search 
whose dynamic environments hinder
reproducibility~\citep{BrowseComp-Plus}, this design enables stable,
fine-grained error attribution. We evaluate eight frontier LLMs
using accuracy 
and
Evidence Hit@$k$
(Appendix~\ref{app:search_impl}).

\subsubsection{Results}

Table~\ref{tab:benchmark_results} presents accuracy under four
evaluation settings for bridge and comparison questions.

\paragraph{Synthesis Quality and Benchmark Difficulty.}
The \textsc{Closed-book} setting establishes the difficulty baseline:
even the best-performing model (GPT-5.4) achieves only 47.2\% on
bridge and 32.3\% on comparison, while most models score below 35\%.
This confirms that the OCG-driven synthesis produces genuinely
challenging cross-domain questions beyond parametric knowledge.

\paragraph{Retrieval Utility and Search Planning.}
The transition from \textsc{Closed-book} to \textsc{1-Search}
yields consistent gains ranging from 0.3 to 37.6 points on bridge.
\textsc{Multi-Search} produces further improvements for most models
(e.g., Gemini~3~Pro +24.9 to 63.0), yet some show degradation
(Gemini~3.1~Pro $-$13.2), indicating that effective iterative search
planning remains challenging for current systems.

\paragraph{Benchmark Discriminative Power.}
Most models exceed 80\% in the \textsc{Oracle} setting (GPT-5.4:
96.1\% on bridge), yet the large \textsc{M-Srch}--\textsc{Oracle}
gap quantifies substantial remaining improvement space for agentic
search strategies. Detailed analysis by association level, temporal
source, and evidence hit rates is in
Appendix~\ref{app:bench_breakdown}.

\section{Conclusion}
\label{sec:Conclusion}

We presented \textsc{Cortex}, to our knowledge the first web-scale
corpus construction framework to unify content quality refinement,
lightweight ontology construction, and cross-domain association
modeling through an Ontological Corpus Graph (OCG). Experiments
validate the effectiveness of the refinement pipeline, the OCG's
domain organization and cross-domain association modeling, and the
data synthesis methodology. 
We publicly release the complete pipeline, a 24.14\,B-token refined
Chinese corpus with its OCG, and \textsc{CortexBench}.
By coupling quality refinement with structured
knowledge organization, \textsc{Cortex} advances corpus construction
beyond flat document filtering, establishing a high-quality,
well-organized ontological corpus graph at web scale.

\section*{Limitations}

The current instantiation processes a single Common Crawl snapshot and
is validated on Chinese; the \textsc{CortexBench} data synthesis
utilizes only 42 of the OCG's over 1.3~million keywords. These scope
choices stem from computational constraints rather than methodological
limitations: each pipeline stage operates independently across
snapshots, and the data synthesis procedure can readily scale to the full
keyword vocabulary. This enables expansion to additional snapshots and
languages, and data synthesis at substantially greater scale for both
evaluation and training.

\bibliography{custom}

\appendix

\section{Pipeline Comparison Table Details}
\label{app:comparison_table}

Table~\ref{tab:comparison_full} presents the full unabbreviated
version of Table~\ref{tab:comparison}. Abbreviations used in the
main-text table: \textbf{FineW}: FineWeb-Edu~\citep{FineWeb};
\textbf{Ne-CC}: Nemotron-CC~\citep{Nemotron-CC};
\textbf{Metadata extr.}: Metadata extraction;
\textbf{Language ID}: Language identification;
\textbf{Q\&A / explan.}: Q\&A / explanation;
\textbf{Expression qua.}: Expression quality;
\textbf{Info.\ density}: Information density;
\textbf{Temporal}: Temporal relevance;
\textbf{Domain class.}: Domain classification;
\textbf{Inter-domain ass.}: Inter-domain association.
Each quality assessment row indicates whether the pipeline's scoring
rubric or filtering mechanism explicitly accounts for that aspect of
content quality. ``OCG'' denotes capabilities provided by the
Ontological Corpus Graph.

\begin{table*}[t]
\centering\footnotesize
\setlength{\tabcolsep}{3.5pt}
\begin{tabular}{@{}ll ccccc@{}}
\toprule
& \textbf{Pipeline Step}
  & \textbf{\textsc{Cortex}}
  & \textbf{FineWeb(-Edu)}
  & \textbf{DCLM}
  & \textbf{C4}
  & \textbf{Nemotron-CC} \\
\midrule
\multirow{5}{*}{\textbf{Preprocessing}}
  & HTML-to-Text
    & \cmark & \cmark & \cmark & $\circ$ & \cmark \\
  & Body extraction
    & \cmark & \cmark & \cmark & \xmark & \cmark \\
  & Metadata extraction
    & \cmark & \xmark & \xmark & \xmark & \xmark \\
  & Language identification
    & \cmark & \cmark & \cmark & \cmark & \cmark \\
  & Deduplication
    & \cmark & \cmark & \cmark & \cmark & \cmark \\
\midrule
\multirow{7}{*}{\shortstack[l]{\textbf{Quality}\\\textbf{Assessment}}}
  & Causal reasoning
    & \cmark & \xmark & \xmark & \xmark & \xmark \\
  & Q\&A / explanation
    & \cmark & \xmark & $\circ$ & \xmark & $\circ$ \\
  & Factual accuracy
    & \cmark & $\circ$ & \xmark & \xmark & $\circ$ \\
  & Expression quality
    & \cmark & $\circ$ & \xmark & $\circ$ & $\circ$ \\
  & Information density
    & \cmark & $\circ$ & \xmark & \xmark & $\circ$ \\
  & Harmful content
    & \cmark & $\circ$ & $\circ$ & \cmark & $\circ$ \\
  & Temporal relevance
    & \cmark & \xmark & \xmark & \xmark & \xmark \\
\midrule
\multirow{2}{*}{\shortstack[l]{\textbf{Cross-domain}\\\textbf{Modeling}}}
  & Domain classification
    & OCG & \xmark & \xmark & \xmark & \xmark \\
  & Inter-domain assoc.
    & OCG & \xmark & \xmark & \xmark & \xmark \\
\bottomrule
\multicolumn{7}{l}{\scriptsize
  \cmark\;explicitly addressed;\quad
  $\circ$\;partially captured;\quad
  \xmark\;not addressed.} \\
\end{tabular}
\caption{Full unabbreviated version of Table~\ref{tab:comparison}:
comparison of \textsc{Cortex} with representative web corpus pipelines
across preprocessing, quality assessment, and cross-domain modeling.}
\label{tab:comparison_full}
\end{table*}

\section{Data Format and Auxiliary Fields}
\label{app:data_format}

After Stage~I, each document is stored as a JSON record containing:
\texttt{url}, \texttt{warc\_id}, \texttt{date} (publication date via
Trafilatura), \texttt{title}, \texttt{language}, \texttt{text\_main}
(main body text for all subsequent processing), and several auxiliary
fields including full-page text, non-body text, multimedia URLs, and
extracted question patterns. A complete field specification is available
in our released codebase.

\section{Preprocessing Pipeline Details}
\label{app:preprocessing_details}

\paragraph{Data License.}
The raw data used in this work is sourced from Common
Crawl~\citep{commoncrawl}, a publicly available web archive that
provides its datasets free of charge for any purpose, 
including
commercial use, under liberal terms of
use.
\footnote{\url{https://commoncrawl.org/terms-of-use}}

\paragraph{HTML Content and Metadata Extraction.}
We extract main body text from raw HTML using
Resiliparse~\cite{Resiliparse,resiliparse2}, a
high-performance web archive parsing toolkit. The DCLM
study~\cite{DCLM} reports that Resiliparse achieves comparable
downstream quality to Trafilatura~\citep{Trafilatura} while
being approximately $8\times$ faster, making it well-suited for
trillion-token-scale processing. Because Resiliparse does not reliably
extract structured metadata (e.g., publication date), we additionally
employ Trafilatura for metadata and main body text extraction. We also
retain several auxiliary data fields; details are in
Appendix~\ref{app:data_format}.

\paragraph{Multi-Stage Language Filtering.}
We adopt a three-stage coarse-to-fine strategy to identify documents in
the target language:
\textbf{(i)}~Resiliparse's built-in language filter performs a rapid first
pass, exploiting its high throughput to efficiently discard obviously
off-target pages at full archive scale;
\textbf{(ii)}~surviving documents are re-classified with the FastText
language identification model~\citep{FastText}, which achieves high
per-language recall on the WiLI-2018 benchmark~\citep{WiLI-2018}
(e.g., $0.984$ for Chinese and $1.000$ for English);
\textbf{(iii)}~since the scoring rubric explicitly down-weights
off-target-language content (\S\ref{sec:teacher}), our document-level
quality scoring model (\S\ref{sec:doc_qa}) inherently penalizes such
documents, serving as an implicit final language purity check.

\paragraph{Global Exact Deduplication.}
Inspired by the snapshot-wide deduplication practices of
FineWeb~\citep{FineWeb} and
Nemotron-CC~\citep{Nemotron-CC}, we perform global exact deduplication
within a single snapshot to remove identical documents across different
WARC slices. The core idea is to compute a deterministic fingerprint for
each document's normalized main body text and retain only the earliest
occurrence of each unique fingerprint. We implement this via a scalable
three-pass \emph{bucket--select--rewrite} algorithm; details are given in
Appendix~\ref{app:dedup}.

\section{Deduplication Algorithm Details}
\label{app:dedup}

Our three-pass deduplication operates as follows.
\textbf{Pass~A} scans all documents, applies a deterministic text
normalization function
$\operatorname{Norm}(\cdot)$---comprising NFKC Unicode normalization,
whitespace/newline folding, zero-width and control character removal, and
leading/trailing stripping---and computes a 128-bit fingerprint
$h(d)=\texttt{xxh3\_128}\!\bigl(\operatorname{Norm}(
\texttt{text\_main}(d))\bigr)$ for each document~$d$.
Fingerprint--position pairs
$\langle h(d),\,\operatorname{pos}(d)\rangle$ are written into
hash-partitioned buckets $b(d)=h(d)\bmod B$, where $B$ is the number of
buckets. Because identical texts always fall into the same bucket,
within-bucket deduplication is equivalent to global deduplication.
\textbf{Pass~B} applies a \emph{keep-first} policy within each bucket:
for each unique fingerprint~$f$, only the document with the earliest
position in lexicographic file order is retained, i.e.,
$d^*_f = \arg\min_{d:\,h(d)=f}\operatorname{pos}(d)$. This produces a
global keep bitmap.
\textbf{Pass~C} rewrites only kept records to a new directory.

\section{Teacher Model Selection and Scoring Rubrics}
\label{app:teacher_models}

\subsection{Teacher Model Selection}
\label{app:model_selection}

While \textsc{Cortex} can be applied to any target language
without architectural changes, we instantiate it on \textbf{Chinese} in
this work.
The teacher models are therefore selected based on their Chinese
language capabilities. We employ an ensemble of $|\mathcal{M}|{=}3$
frontier LLMs:
DeepSeek-V3.1~\citep{Deepseek-v3},
GPT-5~\citep{GPT-5}, and
Doubao-Seed-1.6, selected based on their top rankings
on the ReLE benchmark~\citep{RELE_bench}, a scalable evaluation system
that diagnoses capability anisotropy across 7 domains and
${\sim}$300 fine-grained dimensions for Chinese LLMs.

\paragraph{Teacher Model Scale.}
Among the three teachers, DeepSeek-V3.1 is the only open-weight model;
it inherits the DeepSeek-V3 architecture, a Mixture-of-Experts model
comprising 671\,B total parameters with 37\,B activated per
token~\citep{Deepseek-v3}. GPT-5 and Doubao-Seed-1.6 are
closed-source systems whose parameter counts have not been officially
disclosed. Collectively, each teacher operates at a scale that is
orders of magnitude larger than the 0.3\,B-parameter student, yielding
a compression ratio exceeding $1{,}000\!\times$.

\subsection{Document-Level Teacher Scoring Rubric}
\label{app:prompts}

The document-level teacher scoring rubric instructs each teacher
$\mathcal{M}_m$ to output a JSON object with
\texttt{hard\_label}$\in$\{\textsc{HighQ},\textsc{MidQ},\textsc{LowQ}\},
\texttt{soft\_label}$\in[0,1]$, and
\texttt{time\_related}$\in$\{True, False\}.

The evaluation dimensions are:
(1)~\textbf{Causal reasoning}: presence of causal chains,
argumentation, or step-by-step reasoning;
(2)~\textbf{Question answering / explanation}: clear answers to ``why''
or ``how'' questions;
(3)~\textbf{Factual definitions / descriptions}: accurate scientific,
historical, or conceptual information;
(4)~\textbf{Expression quality}: coherence, fluency, completeness;
(5)~\textbf{Information accuracy}: reliability and correctness.
Content exhibiting strong causal reasoning, high factual accuracy, and
expressive quality receives higher scores.

Additional instructions include explicit down-weighting of
adult/gambling/advertising content, temporal relevance assessment based
on content nature (not metadata), and lower scores for
off-target-language content. The input is constructed as
\texttt{[Publication date:\{date\}]. \{title\}. \{text\_main\}}.
The complete document-level teacher scoring rubric prompt is provided
in the released codebase.

\subsection{Sentence-Level Teacher Scoring Rubric}
\label{app:sen_prompt}

The sentence-level teacher scoring rubric instructs each teacher
$\mathcal{M}_m$ to output a JSON object with
\texttt{hard\_label}$\in$\{\textsc{HighQ\_Sen}, \textsc{MidQ\_Sen},
\textsc{LowQ\_Sen}\} and \texttt{soft\_label}$\in[0,1]$, following the
same threshold mapping as the document level: $[0,\,0.4]$ corresponds to
\textsc{LowQ\_Sen}, $(0.4,\,0.7)$ to \textsc{MidQ\_Sen}, and
$[0.7,\,1.0]$ to \textsc{HighQ\_Sen}. No temporal relevance judgment is
included, as this is resolved at the document level
(\S\ref{sec:doc_qa}).

The rubric inherits the core quality dimensions from the document-level
rubric and additionally emphasizes:
(1)~\textbf{Logical coherence markers}: sentences containing explicit
numbering or structural markers (e.g., sequential markers such as
\emph{firstly}/\emph{secondly}, numbered lists such as
\emph{1.}/\emph{2.}, or enumeration symbols) receive a scoring boost, as
they preserve the overall coherence of the merged paragraph in downstream
processing;
(2)~\textbf{Causal connectives}: sentences with explicit causal
connectives (e.g., \emph{because}, \emph{therefore}, \emph{hence},
\emph{as a result}, \emph{thus it can be seen}) in valid reasoning
receive a significant score increase;
(3)~\textbf{Information density}: substantive, information-rich sentences
are preferred over low-density guiding content;
(4)~\textbf{Paragraph coherence}: transitional and summarizing sentences
that link preceding and following context are scored favorably, balancing
intrinsic quality and structural contribution.
The complete sentence-level teacher scoring rubric prompt is provided
in the released codebase.

\section{Adaptive Sample Weighting}
\label{app:sample_weighting}

Samples near the critical decision boundary and those belonging to the
target positive class may carry more discriminative information. We
apply adjustable weights
$w_n = w_0
  \cdot\bigl(1+\lambda_{+}
         \cdot\mathbb{I}[y_n\in\mathcal{Y}_{+}]\bigr)
  \cdot\bigl(1+\lambda_{\text{bd}}
         \cdot\mathbb{I}[|y_n-t_*|\le\epsilon]\bigr)$,
where $w_0$ is the base weight, $\mathcal{Y}_{+}$ is the target positive
class, $t_*$ is the critical decision boundary, $\lambda_{+}$ and
$\lambda_{\text{bd}}$ are amplification coefficients, and $\epsilon$ is
the boundary bandwidth. For document-level quality scoring,
$\mathcal{Y}_{+}=\{y:y\ge t_2\}$ (i.e., \textsc{HighQ}) and $t_*=t_2$.
For sentence-level scoring,
$\mathcal{Y}_{+}=\{y:y>t_1\}$ (i.e., the ``Keep'' class) and
$t_*=t_1$. Hyperparameter values are listed in
Appendix~\ref{app:hyperparams}.

\section{OAR Temperature Behavior}
\label{app:oar_temperature}

\paragraph{Ordinal Boundary Correspondence.}
The two ordinal boundaries in $\mathcal{L}_{\text{ord}}$ correspond
to the \textsc{LowQ}/\textsc{MidQ} threshold ($t_1$) and the
\textsc{MidQ}/\textsc{HighQ} threshold ($t_2$). Non-negativity of all
three class probabilities ($q_{\text{L}},q_{\text{M}},q_{\text{H}}$)
is guaranteed since $t_2>t_1$ and $\sigma$ is monotonically
increasing, ensuring $c_2 \ge c_1$ for any input score $s$.

In the ordinal boundary loss $\mathcal{L}_{\text{ord}}$
(\S\ref{sec:oar_loss}), the temperatures $\tau$ and $\tau^*$ are
separately tunable for the predicted and target distributions,
respectively. As $\tau\to 0$, the predicted distribution
$\hat{\mathbf{p}}_n$ approaches a hard one-hot vector concentrated on
the class containing $\hat{y}_n$, recovering standard classification.
Larger $\tau$ values produce smoother distributions that permit
non-zero gradients even when $\hat{y}_n$ is far from the class
boundaries, facilitating optimization in early training stages. An
analogous trade-off applies to $\tau^*$ for the target distribution.
In practice, we set $\tau{=}\tau^*{=}0.05$
(Table~\ref{tab:hyperparams}), which provides sharp but differentiable
boundary signals.

\begin{table*}[t]
\centering

\scriptsize
\renewcommand{\arraystretch}{1.1}
\setlength{\tabcolsep}{5pt}
\begin{tabular}{p{0.965\textwidth}}
\toprule
\rowcolor{tableheader}\textcolor{white}{\textbf{Prompt Template for Concept Chain Expansion}} \\
\midrule

\rowcolor{tablesection}\textbf{\# System Prompt} \\
You are a content expansion expert in Chinese academic writing and knowledge graph domains. Your task is to expand a ``concept chain'' into a short paragraph for semantic vector matching.
\par\vspace{1pt}
Requirements: preserve the semantics of the original concept chain, do not introduce irrelevant content, and do not include named entities such as specific works, brands, or people.
\par\vspace{1pt}
The output should be a natural Chinese paragraph of approximately 300--500 Chinese characters. It should cover the definition or core topic, the scope of subtopics, common discussion dimensions or scenarios, and related keywords, which should be naturally integrated into the text rather than presented as a list.
\par\vspace{1pt}
Output only the expanded text, without any prefix, suffix, quotation marks, numbering, or JSON. \\

\midrule
\rowcolor{tablesection}\textbf{\# User Prompt} \\
Concept chain: {\ttfamily \{chain\}}
\newline Please expand it according to the above requirements. \\

\bottomrule
\end{tabular}

\caption{Prompt template for concept chain expansion. The prompt guides the LLM to expand a concept chain into a Chinese paragraph suitable for semantic vector matching while preserving the original chain semantics and avoiding named entities.}
\label{tab:concept_chain_expansion_prompt}
\vspace{-1.2\baselineskip}

\end{table*}

\section{Student Model Architecture}
\label{app:student_arch}
\label{sec:student_arch}

All student models share a unified architecture $\mathcal{S}_\theta$,
parameterized by $\theta$, initialized from
MacBERT-Large~\citep{MacBERT} (${\sim}$0.3\,B parameters). A
single linear projection head maps the encoded representation to a
scalar logit, and a sigmoid function produces the predicted score
$\hat{y}=\mathcal{S}_\theta(x)=\sigma(\tilde{z})
=1/(1+e^{-\tilde{z}})$. The output interpretation differs by task: the
temporal classifier produces a binary probability, while quality scorers
produce a continuous score in $[0,1]$ mapped to ordinal labels through a
post-hoc decision threshold. Since document-level tasks require
processing texts far exceeding MacBERT's 512-token context window, we
introduce a \emph{chunk-then-aggregate} architecture; sentence-level
scoring uses the encoder directly without chunking.

\paragraph{Chunk Encoding.}
Given a document with main body text $\mathbf{x}$, we prepend the
publication date and apply a sliding window of length $L{=}512$ with
stride $S{=}256$ to produce chunks $\{x_1,\ldots,x_J\}$, retaining
only the first $J_{\max}$ chunks (head selection). Each chunk is
independently encoded by MacBERT; its pooler output $\mathbf{h}_j$ is
projected to a scalar logit
$z_j = \mathbf{w}^\top\!\operatorname{Dropout}(\mathbf{h}_j)+b$ for
$j=1,\ldots,J$.

\paragraph{Document-Level Aggregation.}
We aggregate chunk logits into a single document logit $\tilde{z}$ via
one of two strategies depending on the task.
For \emph{temporal relevance classification}, we use the \textbf{mean}:
$\tilde{z}^{\,\text{mean}} = \frac{1}{J}\sum_{j=1}^{J} z_j$.
For \emph{quality scoring}, we use a \textbf{proportional top-$k$}
(top-$k_\rho$) strategy that focuses on the most informative segments:
given a ratio $\rho\in(0,1]$, the number of aggregated chunks is
$K = \min(J,\;\max(1,\lceil\rho J\rceil))$, and the document logit is
the mean of the $K$ highest-valued chunk logits:
$\tilde{z}^{\,\text{topk}} = \frac{1}{K}\sum_{i=1}^{K} z_{\pi(i)}$,
where $\pi$ sorts $\{z_j\}_{j=1}^J$ in descending order.

\section{Training Protocol and Data Allocation}
\label{app:training}

\subsection{Training Protocol}
\label{app:training_protocol}

For all student models, a natural-distribution test set is held out from
the teacher-annotated data before constructing balanced training sets.
All models are trained with 5-fold stratified cross-validation as a
\emph{model selection} mechanism (not an ensemble): the fold yielding the
best validation metric is selected as the final model. Early stopping
monitors validation macro-$F_1$ for the temporal classifier and the
target-class $F_2$ for quality scorers, with a patience of 2 epochs.

\subsection{Data Allocation}
\label{app:data_split}

\paragraph{Document-Level.}
The teacher-annotated set $\mathcal{D}_{\text{anno}}$ contains 60,626
documents at the document level. A 5\% natural-distribution test set
(3,031 samples) is held out first. From the remaining 57,595 samples,
three balanced subsets are drawn:
the \emph{temporal classifier} set contains 8,220 samples (4,110
$\textsc{Time\text{-}Rel}$ and 4,110 $\textsc{Time\text{-}Irrel}$;
within each temporal category, samples are uniformly distributed across
the three quality
tiers---\textsc{HighQ}, \textsc{MidQ}, and \textsc{LowQ}---at 1,370 per
tier, to prevent confusion between temporal and quality signals);
the $\textsc{Time\text{-}Rel}$ quality scorer
$\mathcal{S}_{\text{dq-tr}}^*$ uses 4,110 samples exclusively from
$\textsc{Time\text{-}Rel}$ documents (1,370 per quality tier:
\textsc{HighQ}, \textsc{MidQ}, \textsc{LowQ});
the $\textsc{Time\text{-}Irrel}$ quality scorer
$\mathcal{S}_{\text{dq-ti}}^*$ uses 9,939 samples exclusively from
$\textsc{Time\text{-}Irrel}$ documents (3,313 per quality tier:
\textsc{HighQ}, \textsc{MidQ}, \textsc{LowQ}).

\paragraph{Sentence-Level.}
At the sentence level, 57,024 sentences from selected documents in
$\mathcal{D}_0$ are annotated by $\mathcal{M}$, from which a 5\%
natural-distribution test set (2,852 sentences) is held out. The
sentence-level quality scorer $\mathcal{S}_{\text{sq}}^*$ is trained on
22,230 balanced samples (7,410 per quality tier: \textsc{HighQ\_Sen},
\textsc{MidQ\_Sen}, and \textsc{LowQ\_Sen}). Note that sentence-level
quality scoring does not involve temporal categorization, as temporal
relevance is resolved at the document level (\S\ref{sec:doc_qa}).

\subsection{Evaluation Metrics}
\label{app:eval_metrics}

Since the document-level pipeline retains only \textsc{HighQ} documents,
we prioritize recall of the \textsc{HighQ} class and adopt $F_2$ as the
primary evaluation metric, which weights recall twice as heavily as
precision. For the temporal classifier, we use macro-$F_1$---the
unweighted average of per-class $F_1$ scores across
$\textsc{Time\text{-}Rel}$ and
$\textsc{Time\text{-}Irrel}$---to ensure balanced evaluation of both
precision and recall across both temporal categories. For sentence-level
scoring, the target is the Keep class
(\textsc{MidQ\_Sen}$\cup$\textsc{HighQ\_Sen}), and we similarly use
Keep-$F_2$.

\subsection{Pipeline Data Flow}
\label{app:pipeline_stats}

We instantiate \textsc{Cortex} on Chinese using the
Common Crawl~\citep{commoncrawl} snapshot CC-MAIN-2025-43
(${\sim}$2.61\,B raw pages).
Table~\ref{tab:pipeline_stats} traces data
volume through each stage, following the notation in
\S\ref{sec:methodology}.

\begin{table}[t]
\centering\small
\begin{tabular}{lr}
\toprule
\textbf{Stage} & \textbf{Volume}\\
\midrule
Raw CC snapshot (pages) & ${\sim}$2.61\,B\\
After coarse language filter & 118.94\,M\\
After fine language filter (zh) & 89.26\,M\\
$\mathcal{D}_0$ (after global exact dedup) & 73.15\,M\\
$\mathcal{D}_{\text{dhq}}$ (after doc-level filter) & 15.13\,M\\
\quad $\textsc{Time\text{-}Irrel}$ & 9.67\,M\\
\quad $\textsc{Time\text{-}Rel}$ & 5.46\,M\\
\midrule
\multicolumn{2}{l}{\emph{Token counts (main body):}}\\
$\mathcal{D}_{\text{dhq}}$ & 27.18\,B\\
$\mathcal{D}_{\text{refined}}$ (after sentence refinement)
  & \textbf{24.14\,B}\\
\quad $\textsc{Time\text{-}Irrel}$ & 15.05\,B\\
\quad $\textsc{Time\text{-}Rel}$ & 9.09\,B\\
\bottomrule
\end{tabular}
\caption{Data volume at each pipeline stage (Chinese instantiation on
Common Crawl~\citep{commoncrawl} snapshot CC-MAIN-2025-43). Symbols
follow \S\ref{sec:methodology}.}
\label{tab:pipeline_stats}
\end{table}

\paragraph{Released Resources.}
We release the complete \textsc{Cortex} pipeline, the trained student
scoring models, the 24.14\,B-token refined corpus with its OCG, and
\textsc{CortexBench}, providing the research community with a
reproducible end-to-end pipeline, quality-refined training data with
structured knowledge organization, and a rigorous cross-domain
evaluation benchmark.

\section{Threshold Calibration Protocol}
\label{app:threshold}

\paragraph{Structural vs.\ Decision Thresholds.}
The ordinal boundaries $t_1,t_2$ in $\mathcal{L}_{\text{OAR}}$ are
\emph{fixed structural parameters} governing the training objective.
At inference time, separate \emph{post-hoc decision thresholds}
$\eta_{\text{doc}}$ and $\eta_{\text{sen}}$ are calibrated on held-out
natural-distribution sets for document-level and sentence-level
scoring, respectively, to control the precision--recall trade-off.
This decoupling allows the same trained model to be deployed with
different thresholds for different applications.

Each natural-distribution test set is split into a \emph{calibration}
(40\%) and \emph{final test} (60\%) partition. On the calibration set, we
search for the task-specific decision threshold
($\eta_{\text{doc}}$ for document-level quality scorers, $\eta_{\text{sen}}$ for the sentence-level scorer) that
maximizes the target-class $F_\beta$ score:
\begin{equation}
  F_\beta(\theta)
    = \frac{(1{+}\beta^2)\,P(\theta)\,R(\theta)}
           {\beta^2 P(\theta)+R(\theta)},
\end{equation}
where $P$ and $R$ are precision and recall at threshold $\theta$. For
document-level quality scorers, the target class is \textsc{HighQ} with
$\beta{=}2$; for sentence-level scoring, the target class is Keep
(\textsc{MidQ\_Sen}$\cup$\textsc{HighQ\_Sen}) with $\beta{=}2$. The
temporal classifier uses a fixed threshold of $0.5$ without calibration.
The selected thresholds are frozen for final-test evaluation and
full-corpus inference.

\section{Corpus-Scale Inference Details}
\label{app:inference_details}

\paragraph{Temporal-Specific Scoring.}
The document-level quality score $\hat{y}_{\text{dq}}(d)$ is produced
by the task-appropriate scorer:
$\hat{y}_{\text{dq}}(d) = \mathcal{S}_{\text{dq-tr}}^*(d)$ for
$d\in\mathcal{D}_0^{\text{TR}}$ and
$\hat{y}_{\text{dq}}(d) = \mathcal{S}_{\text{dq-ti}}^*(d)$ for
$d\in\mathcal{D}_0^{\text{TI}}$. This bifurcation accounts for the
differing quality criteria between time-sensitive and time-invariant
content.

\paragraph{Sentence Segmentation.}
Each document $d\in\mathcal{D}_{\text{dhq}}$ is segmented into
sentences using Stanza~\cite{Stanza}, a neural pipeline supporting
tokenization, POS tagging, and syntactic parsing for over 70
languages.

\section{Hyperparameter Settings}
\label{app:hyperparams}

Table~\ref{tab:hyperparams} summarizes student model hyperparameters.

\begin{table}[t]
\centering\small
\setlength{\tabcolsep}{3pt}
\begin{tabular}{lcccc}
\toprule
\textbf{Parameter} & \textbf{Temp.} & \textbf{DQ-TR} & \textbf{DQ-TI}
  & \textbf{SQ}\\
\midrule
Backbone & \multicolumn{4}{c}{MacBERT-Large (${\sim}$0.3B)}\\
Optimizer & \multicolumn{4}{c}{AdamW (lr=2e-5, wd=0.01)}\\
Max seq len & 512 & 512 & 512 & 512\\
Stride & 256 & 256 & 256 & ---\\
Chunking & Yes & Yes & Yes & No\\
Chunk repr. & pooler & pooler & pooler & pooler\\
$J_{\max}$ (train/infer) & 14 & 14 & 14 & ---\\
Aggregation & mean & top-$k_\rho$ & top-$k_\rho$ & ---\\
$\rho$ & --- & 0.2 & 0.2 & ---\\
Loss & BCE & OAR & OAR & OAR\\
$\lambda_{\text{reg}}/\lambda_{\text{ord}}$ & --- & 1.0/1.0 & 1.0/1.0
  & 1.0/1.0\\
$\tau/\tau^*$ & --- & .05/.05 & .05/.05 & .05/.05\\
Huber $\delta$ & --- & 0.1 & 0.1 & 0.1\\
$\lambda_{+}$ & --- & 1.0 & 1.0 & 1.0\\
$\lambda_{\text{bd}}$ & --- & 1.0 & 1.0 & 1.0\\
$\epsilon$ & --- & 0.05 & 0.05 & 0.05\\
Epochs & 8 & 10 & 10 & 10\\
Batch / Grad accum & 8/8 & 4/8 & 4/8 & 4/8\\
Precision & bf16 & bf16 & bf16 & bf16\\
Warmup ratio & 0.06 & 0.06 & 0.06 & 0.06\\
Dropout & 0.1 & 0.1 & 0.1 & 0.1\\
\midrule
Early stop metric
  & macro-$F_1$
  & \multicolumn{2}{c}{HighQ-$F_2$}
  & Keep-$F_2$\\
Early stop patience
  & \multicolumn{4}{c}{2}\\
Cross-validation
  & \multicolumn{4}{c}{5-fold stratified (model selection)}\\
\bottomrule
\end{tabular}
\caption{Student model hyperparameters. \textbf{Temp.}: temporal
relevance classifier $\mathcal{S}_{\text{tc}}^*$;
\textbf{DQ-TR}/\textbf{DQ-TI}: document-level quality scorer
$\mathcal{S}_{\text{dq-tr}}^*$/$\mathcal{S}_{\text{dq-ti}}^*$;
\textbf{SQ}: sentence-level quality scorer
$\mathcal{S}_{\text{sq}}^*$.}
\label{tab:hyperparams}
\end{table}

\paragraph{OCG Construction.}
Embedding model: BGE-M3~\citep{bge-m3} (0.56B, multilingual).
Keyword fusion: $\alpha_s{=}\beta_s{=}0.5$; top-$T{=}5$ per document.
Document--concept: top-$M{=}3$ chains. Keyword--chain: top-1 per
keyword.
Evolution LLM $\mathcal{M}_{\text{evo}}$:
Doubao-Seed-2.0-Pro
(\texttt{doubao-seed-2-0-pro-260215}).
Evolution: $W{=}2{,}000$; patience $P{=}300$;
$\tau_{\text{conf}}{=}0.80$; $\tau_{\text{admit}}{=}0.90$;
$w_1{=}0.6$, $w_2{=}0.4$.

\section{Computational Resources}
\label{app:compute}

All continual pre-training experiments (\S\ref{sec:exp_rq2}) are
conducted on a cluster equipped with
48*NVIDIA H100 80\,GB HBM3 GPUs, Intel Xeon Platinum 8468 CPUs,
running Rocky Linux 8.6 (Green Obsidian). All other pipeline
stages, including data preprocessing, knowledge distillation
training and corpus-scale inference, OCG construction, and
\textsc{CortexBench} synthesis and evaluation, are performed on a
separate environment with three NVIDIA A800 80\,GB PCIe GPUs, Intel
Xeon Gold 6326 CPUs at 2.90\,GHz, running Ubuntu 20.04.6 LTS.

\section{Concept Chain Evolution Details}
\label{app:concept_details}

\subsection{Initial Concept Chain Set Structure}
\label{app:concept_hierarchy}

The initial Concept Chain Set $\mathcal{C}^{(0)}$ is organized as a
multi-level domain hierarchy, where each concept chain represents a
root-to-leaf path. Chains may have varying depths (typically 2--4
levels), reflecting different domains' natural granularity. For example,
\texttt{Natural Science>Physics>Optics} has three levels, while
\texttt{Sports>Basketball} has two. The hierarchy spans 12 top-level
domains (Table~\ref{tab:taxonomy_stats}), containing 196 leaf-level
concept chains in total.

A defining property of this structure is that every node at every level
of the hierarchy---from the root through intermediate nodes to leaf
chains---constitutes a valid domain identifier. Leaf-level nodes
represent the finest-grained domains, while higher-level nodes represent
progressively broader domains at coarser granularity. For instance,
given the chain \texttt{Natural Science>Physics>Optics}, the leaf node
\texttt{Optics} identifies a fine-grained domain, the intermediate node
\texttt{Natural Science>Physics} a broader one, and the root
\texttt{Natural Science} the broadest. This multi-resolution property is
directly leveraged in the inter-domain association computation
(\S\ref{sec:graph_construct}), enabling domain relationship analysis at
arbitrary resolution without any structural modification.

\subsection{Matching and Evolution Protocol}
\label{app:concept_prompt}

The evolution procedure instructs $\mathcal{M}_{\text{evo}}$ to:
(1)~extract 1--3 core domain topics from each document $d_t$;
(2)~match each to existing chains in the current Concept Chain Set
$\mathcal{C}^{(t-1)}$ (using full root-to-leaf paths), with a quality
baseline of confidence $\ge\tau_{\text{conf}}$ for normal matching;
(3)~propose new leaf-level chains \emph{only when no existing chain
reaches} $\tau_{\text{conf}}$, subject to the structural constraints
specified in \S\ref{app:evolution_constraints};
(4)~output JSON with \texttt{topics} and \texttt{results} (each
containing \texttt{chain}, \texttt{confidence}, \texttt{is\_new},
\texttt{reason}).
The current complete Concept Chain Set $\mathcal{C}^{(t-1)}$ is provided
to $\mathcal{M}_{\text{evo}}$ for matching and judgment. A five-step
internal workflow---semantic extraction, chain matching, quality
assessment, new chain addition (triggered only when needed), and ranked
output---guides $\mathcal{M}_{\text{evo}}$ through a structured
reasoning process. Matching existing chains is always the priority,
governed by principle \textbf{P1} (Existing Chain Priority), with new
chain proposals following the six generation rules of principle
\textbf{P2} (Appendix~\ref{app:evolution_constraints}). Newly added
chains are immediately available in subsequent evolution steps. 
The
complete evolution prompt is provided in the released codebase.

\subsection{Full Structural Constraints}
\label{app:evolution_constraints}

Candidate new chains must satisfy six constraints:
\textbf{(i)}~\emph{Domain abstractness} (highest weight): chains
represent abstract domains generalizing a class of things or a knowledge
domain, and must not refer to any specific named entity, including person
names, place names, organization names, brand names, product models, or
specific event names;
\textbf{(ii)}~\emph{Upper-level stability}: the top two levels of the
hierarchy remain unchanged unless absolutely necessary to accommodate a
fundamentally new domain;
\textbf{(iii)}~\emph{Leaf-level addition only}: new chains must append a
new leaf-level concept under an existing parent path, forming a complete
root-to-leaf path; it is strictly forbidden to use a path truncated at
an intermediate level as a new chain;
\textbf{(iv)}~\emph{Clear parent attachment}: the parent path must be
unambiguous, with the new chain belonging to exactly one logically clear
existing parent node;
\textbf{(v)}~\emph{Sibling distinguishability}: clear semantic
boundaries must exist between the new chain and its sibling concepts at
the same level; synonymous or near-synonymous concepts are forbidden;
\textbf{(vi)}~\emph{Granularity alignment}: the abstraction level of the
new chain must match that of its sibling nodes, being neither too coarse
nor too fine.

\subsection{Concept Expansion Protocol}
\label{app:expansion_prompt}

Each concept chain is expanded by an LLM into a natural-language Chinese
paragraph of approximately 300--500 characters, covering definitions or
core topics, the scope of subtopics, common discussion dimensions or
scenarios, and related keywords. The system prompt explicitly forbids
named entities (specific works, brands, people) and requires keywords to
be naturally integrated into the text rather than presented as a list.
The complete concept expansion prompt is presented in
Table~\ref{tab:concept_chain_expansion_prompt}.

\subsection{Convergence Monitoring}
\label{app:convergence}

Convergence is monitored via the sliding-window update rate $r_t$,
defined as the fraction of the most recent $W$ documents that triggered a
new chain proposal at evolution step $t$:
\begin{equation}
  r_t = \frac{1}{\min(t,W)}
    \sum_{j=\max(1,\,t-W+1)}^{t}
      \mathbb{I}\bigl[\mathcal{C}^{(j)}\neq\mathcal{C}^{(j-1)}\bigr].
  \label{eq:update_rate}
\end{equation}
We set $W{=}5{,}000$.
Evolution terminates under an early stopping criterion: when no proposal
occurs for $P$ consecutive documents (i.e., $r_t{=}0$ persists for a
window of size $P$), the Concept Chain Set is considered stabilized.
Figure~\ref{fig:convergence} plots $r_t$ over the evolution process; the
update rate exhibits an overall decreasing trend, confirming progressive
stabilization of the taxonomy.

\begin{figure}[t]
\centering
\includegraphics[width=\columnwidth]{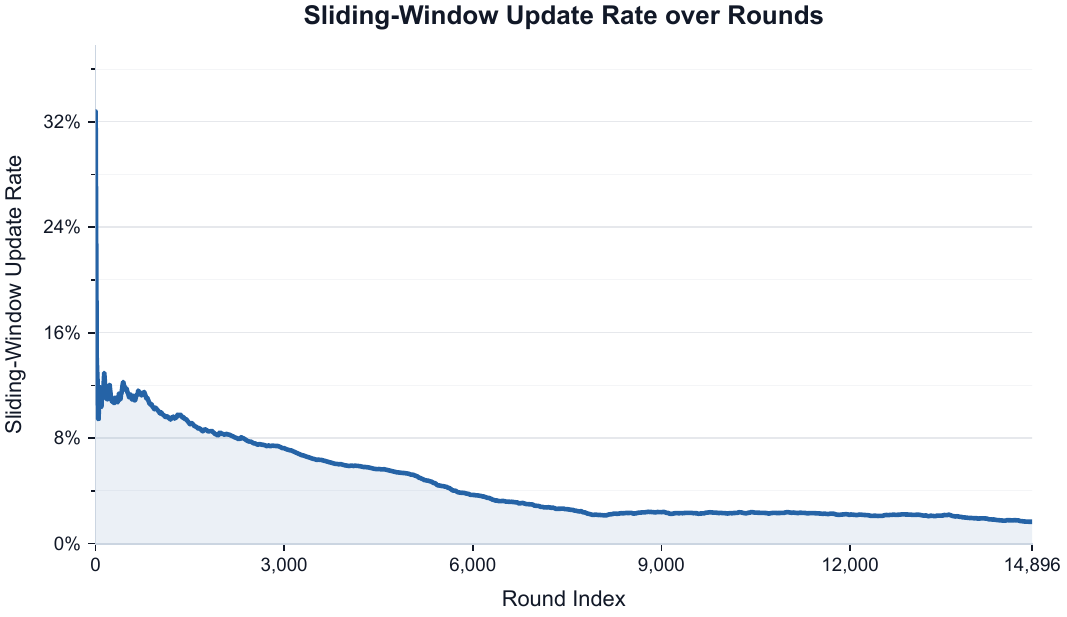}
\caption{Sliding-window update rate $r_t$
(Eq.~\ref{eq:update_rate}) during concept chain evolution.}
\label{fig:convergence}
\end{figure}

\section{Keyword Extraction Details}
\label{app:keyword_details}

Given a document $d\in\mathcal{D}_{\text{refined}}$, keyword extraction
proceeds through four stages: preprocessing, independent scoring by two
methods, score fusion, and top-$T$ selection.

\paragraph{Preprocessing.}
Traditional Chinese text is first converted to Simplified Chinese.
Documents undergo word segmentation and part-of-speech (POS) tagging via
jieba, a dictionary-based Chinese segmentation
toolkit that combines prefix dictionary scanning with an HMM-based model
for unknown word recognition. Stopwords are removed using a consolidated
Chinese stopword list~\citep{chinesestopwords} (merging Baidu, HIT, and
SCU resources).

\paragraph{POS Filtering.}
Let $\mathcal{W}_d^{\text{raw}}=(w_1,w_2,\ldots,w_{L_d})$ denote the
segmented token sequence of document $d$ with corresponding POS tags
$(p_1,p_2,\ldots,p_{L_d})$. We retain only tokens whose POS tag belongs
to the allowed set $\mathcal{P}_{\text{allow}}=\{$\texttt{n},
\texttt{nr}, \texttt{nrt}, \texttt{ns}, \texttt{nt}, \texttt{nz},
\texttt{ng}, \texttt{vn}, \texttt{an}, \texttt{eng}, \texttt{i},
\texttt{l}$\}$, producing the filtered token sequence
$\mathcal{W}_d = (w_{i_1},w_{i_2},\ldots,w_{i_{L'_d}})$ where
$p_{i_j}\in\mathcal{P}_{\text{allow}}$ for all $j$, and $L'_d$ is the
filtered length.

\paragraph{TF-IDF Scoring.}
For each unique keyword $k$ in $\mathcal{W}_d$, let
$f_{k,d}=|\{j: w_{i_j}=k\}|$ denote its frequency in the filtered
sequence. The raw TF-IDF score is:
\begin{equation}
  s_{\text{tfidf}}^{\text{raw}}(k,d)
    = \underbrace{\frac{f_{k,d}}{L'_d}}_{\text{TF}(k,d)}
      \;\cdot\;
      \underbrace{\log\frac{|\mathcal{D}_{\text{idf}}|+1}
                           {1+|\{d'\in\mathcal{D}_{\text{idf}}:
                             k\in d'\}|}}_%
      {\text{IDF}(k)},
  \label{eq:tfidf}
\end{equation}
where $\mathcal{D}_{\text{idf}}$ is a reference corpus from which IDF
values are pre-computed.

\paragraph{TextRank Scoring.}
Following TextRank~\cite{TextRank}, we construct an undirected
co-occurrence graph $G_d=(V_d,E_d)$ from $\mathcal{W}_d$: $V_d$
contains all unique keywords in $\mathcal{W}_d$, and an edge
$(k_i,k_j)\in E_d$ exists if $k_i$ and $k_j$ co-occur within a sliding
window of size $\omega{=}5$ tokens in $\mathcal{W}_d$. The TextRank
score is computed iteratively:
\begin{equation}
  \operatorname{TR}^{(l+1)}\!(k_i)
    = (1{-}\alpha) + \alpha
      \!\!\sum_{k_j\in\mathcal{N}(k_i)}\!\!
      \frac{1}{|\mathcal{N}(k_j)|}
      \operatorname{TR}^{(l)}\!(k_j),
  \label{eq:textrank}
\end{equation}
where $\alpha{=}0.85$ is the damping factor, $\mathcal{N}(k_i)$ denotes
the neighbor set of $k_i$ in $G_d$, and
$\operatorname{TR}^{(0)}(k_i)=1$ for all $k_i$. Iteration proceeds
until convergence
($\max_i|\operatorname{TR}^{(l+1)}(k_i)-
\operatorname{TR}^{(l)}(k_i)| <10^{-4}$), yielding
$s_{\text{tr}}^{\text{raw}}(k,d) = \operatorname{TR}(k)$.

\paragraph{Score Fusion and Selection.}
Each scoring method independently produces a ranked keyword list. We
apply per-document min-max normalization to each method:
\begin{equation}
  s_{\text{tfidf}}(k)
    = \frac{s_{\text{tfidf}}^{\text{raw}}(k)
      - \min_{k'} s_{\text{tfidf}}^{\text{raw}}(k')}
      {\max_{k'} s_{\text{tfidf}}^{\text{raw}}(k')
      - \min_{k'} s_{\text{tfidf}}^{\text{raw}}(k')},
  \label{eq:minmax_tfidf}
\end{equation}
and analogously for TextRank, yielding $s_{\text{tr}}(k)\in[0,1]$. The
fused score is:
\begin{equation}
  s_{\text{comb}}(k)
    = \alpha_s\,s_{\text{tfidf}}(k)
    + \beta_s\,s_{\text{tr}}(k),
  \label{eq:fusion}
\end{equation}
with $\alpha_s{=}\beta_s{=}0.5$. The final keyword set of document $d$
is obtained by selecting the top-$T$ keywords by fused score:
\begin{equation}
  \mathcal{K}_d = \operatorname{top\text{-}}T_{k\in V_d}\;
    s_{\text{comb}}(k),
  \label{eq:topT}
\end{equation}
with $T{=}5$. These keywords are then linked to concept chains via the keyword--concept linking formula (\S\ref{sec:keyword}).

\section{Properties of $\operatorname{Assoc}$}
\label{app:proofs}

\paragraph{Summation vs.\ Averaging.}
The accumulated signal vector $\mathbf{s}_p$ uses summation rather
than averaging over descendant leaves: a keyword connected to many
descendants naturally receives a stronger aggregated signal,
reflecting its broader relevance to the node's overall domain. If
averaging were used instead, a keyword appearing in a single
descendant would receive the same weight as one appearing across all
descendants, obscuring the breadth of its domain relevance.

\paragraph{Neighbor Definition.}
Two nodes $p_i$ and $p_j$ are \emph{neighbors} if their descendant
leaves share at least one active keyword, i.e.,
$\sum_k S_{p_i,k}\,S_{p_j,k}>0$.

\subsection{Leaf-Level Specialization and Degeneracy Consistency}
\label{app:proof_consistency}

When both nodes $p_i{=}c_i$ and $p_j{=}c_j$ are leaf-level concept
chains,
$\mathcal{C}_{c_i}^{\downarrow}=\{c_i\}$, hence
$\mathbf{s}_{c_i}=\mathbf{w}_{c_i}$ and $S_{c_i,k}=W_{c_i,k}$ for all
$k$. The generalized association (Eq.~\ref{eq:assoc}) reduces to:
\begin{equation}
  \operatorname{Assoc}(c_i,c_j)
    = \frac{\sum\nolimits_{k\in\mathcal{K}_{i,j}}
            W_{c_i,k}\,W_{c_j,k}}
           {\|\mathbf{w}_{c_i}\|\;\|\mathbf{w}_{c_j}\|}.
  \label{eq:leaf_sim}
\end{equation}
Two leaf chains are \emph{neighbors} if
$\mathcal{K}_{i,j}\neq\emptyset$.

\begin{proof}
Since $\mathcal{C}_{c_i}^{\downarrow}=\{c_i\}$, substituting into
Eq.~\eqref{eq:assoc} gives
$\operatorname{Assoc}(c_i,c_j)=
\frac{\sum_{k\in\mathcal{K}}W_{c_i,k}\,W_{c_j,k}}
     {\|\mathbf{w}_{c_i}\|\;\|\mathbf{w}_{c_j}\|}$.
Since $W_{c_i,k}\,W_{c_j,k}=0$ for all
$k\notin\mathcal{K}_{i,j}$, the sum over $\mathcal{K}$ collapses to the
sum over $\mathcal{K}_{i,j}$, yielding Eq.~\eqref{eq:leaf_sim}.
\end{proof}

\subsection{Pairwise Decomposition}
\label{app:proof_decomposition}

\begin{proposition}
The numerator of $\operatorname{Assoc}(p_i,p_j)$ decomposes as
\begin{equation}
  \begin{split}
    \sum_{k\in\mathcal{K}}S_{p_i,k}\,S_{p_j,k}
      &= \sum_{c\in\mathcal{C}_{p_i}^{\downarrow}}
         \sum_{c'\in\mathcal{C}_{p_j}^{\downarrow}}
         \!\!\sum_{k\in\mathcal{K}_c\cap\mathcal{K}_{c'}}
         \!\!W_{c,k}\,W_{c',k}.
  \end{split}
\end{equation}
\end{proposition}
\begin{proof}
By expanding the definitions:
$\sum_k S_{p_i,k}\,S_{p_j,k}
 =\sum_k\bigl(\sum_{c\in\mathcal{C}_{p_i}^{\downarrow}}
 W_{c,k}\bigr)
        \bigl(\sum_{c'\in\mathcal{C}_{p_j}^{\downarrow}}
        W_{c',k}\bigr)
 =\sum_{c}\sum_{c'}\sum_k W_{c,k}\,W_{c',k}$,
where the last equality follows from distributivity of multiplication
over addition. Only terms with
$k\in\mathcal{K}_c\cap\mathcal{K}_{c'}$ are non-zero, since
$W_{c,k}=0$ for $k\notin\mathcal{K}_c$.
\end{proof}

Each leaf-pair contribution
$\sum_{k\in\mathcal{K}_c\cap\mathcal{K}_{c'}}W_{c,k}\,W_{c',k}$ enters
the numerator without compression, ensuring full signal fidelity.

\subsection{Worked Example}
\label{app:assoc_example}

Consider computing the association between two intermediate-level nodes
$p_i$ and $p_j$, where:
\begin{align*}
p_i &= \text{Natural Science > Physics},\\
p_j &= \text{Technology\&Engineering > Engineering}.
\end{align*}
Suppose $p_i$ has two descendant leaf chains and $p_j$ likewise:
\begin{align*}
c_1 &= \text{Natural Science > Physics > Optics},\\
c_2 &= \text{Natural Science > Physics > Mechanics},\\
c_3 &= \text{Technology\&Engineering > Engineering}\\
    &\quad\text{> Civil Engineering},\\
c_4 &= \text{Technology\&Engineering > Engineering}\\
    &\quad\text{> Electrical Engineering}.
\end{align*}

The accumulated signal vectors are
$\mathbf{s}_{p_i}{=}\mathbf{w}_{c_1}{+}\mathbf{w}_{c_2}$ and
$\mathbf{s}_{p_j}{=}\mathbf{w}_{c_3}{+}\mathbf{w}_{c_4}$. By the
pairwise decomposition
(Appendix~\ref{app:proof_decomposition}), the numerator of
$\operatorname{Assoc}(p_i,p_j)$ decomposes into four leaf-pair dot
products:
$\mathbf{w}_{c_1}{\cdot}\mathbf{w}_{c_3}$,
$\mathbf{w}_{c_1}{\cdot}\mathbf{w}_{c_4}$,
$\mathbf{w}_{c_2}{\cdot}\mathbf{w}_{c_3}$, and
$\mathbf{w}_{c_2}{\cdot}\mathbf{w}_{c_4}$.

If, for instance, $c_2$ (Mechanics) and $c_3$ (Civil Engineering) share
keywords that carry high TFS weights on both chains---i.e., keywords
such as ``structural analysis,'' ``materials,'' and ``stress'' that are
strongly linked to both domains via large $W_{c_2,k}$ and $W_{c_3,k}$
values---their dot product
$\mathbf{w}_{c_2}{\cdot}\mathbf{w}_{c_3}
=\sum_{k\in\mathcal{K}_{c_2}\cap\mathcal{K}_{c_3}}
W_{c_2,k}\,W_{c_3,k}$ will dominate the sum. This demonstrates that
inter-domain association is determined not merely by the number of shared
keywords, but by the strength of each keyword's association with both
chains (reflected in the TFS weight products). The resulting high
$\operatorname{Assoc}(p_i,p_j)$ value is thus primarily driven by the
Mechanics--Civil Engineering overlap, providing directly interpretable
inter-domain semantics.

At the leaf level (Appendix~\ref{app:proof_consistency}), the
association between any two leaf chains reduces to the cosine similarity
between their TFS weight vectors.

\section{Concept Chain Taxonomy}
\label{app:taxonomy}

Table~\ref{tab:taxonomy_stats} summarizes the distribution of the 403
leaf-level concept chains across top-level domains after evolution and
admission control.

\begin{table}[t]
\centering\small
\begin{tabular}{lr}
\toprule
\textbf{Top-Level Domain} & \textbf{Leaf Chains}\\
\midrule
Culture \& Life & 34\\
Arts \& Entertainment & 37\\
Sports & 11\\
Travel \& Geography & 7\\
Health \& Medicine & 28\\
History & 8\\
Natural Science & 28\\
Philosophy \& Religion & 12\\
Society & 83\\
Business \& Economy & 44\\
Technology \& Engineering & 100\\
People \& Biography & 11\\
\midrule
\textbf{Total} & \textbf{403}\\
\bottomrule
\end{tabular}
\caption{Leaf chain distribution across the 12 top-level domains after
evolution and admission control.}
\label{tab:taxonomy_stats}
\end{table}

The full initial Concept Chain Set $\mathcal{C}^{(0)}$ (196 leaf-level
chains) and the post-evolution set $\mathcal{C}$ (403 leaf-level
chains) are released in our codebase and data package, in both
hierarchical and enumerated representations.

\section{Student Model Evaluation Details}
\label{app:kd_detail}

\paragraph{Metric Selection.}
For the temporal classifier $\mathcal{S}_{\text{tc}}^*$, macro-$F_1$
is used as the primary metric to ensure balanced evaluation across both
\textsc{Time\text{-}Rel} and \textsc{Time\text{-}Irrel}.
For the two document-level quality scorers
$\mathcal{S}_{\text{dq-tr}}^*$ and $\mathcal{S}_{\text{dq-ti}}^*$,
the target class is \textsc{HighQ}; we use $F_2$ (weighting recall
twice as heavily as precision) because the pipeline's goal is to
retain as many high-quality documents as possible. For the
sentence-level quality scorer $\mathcal{S}_{\text{sq}}^*$, the target
class is \textsc{Keep}
(\textsc{HighQ\_Sen}\,$\cup$\,\textsc{MidQ\_Sen}); we likewise use
Keep-$F_2$ to maximize retention of informative sentences.
Spearman rank correlation ($\rho$) between student predictions and
teacher consensus soft scores is reported for the three quality scoring
tasks to assess whether the student preserves the fine-grained ordinal
ranking produced by the teacher ensemble. The consistently high
$\rho$ values (0.805, 0.922, and 0.908) confirm that the OAR loss
(\S\ref{sec:oar_loss}) effectively trains the student to reproduce both
the teacher's classification decisions and the underlying continuous
scoring standard. The calibrated decision thresholds
(Appendix~\ref{app:threshold}) are then applied to these well-ordered
scores for binary retention decisions during corpus-scale inference.

\section{Corpus Distribution Analysis}
\label{app:corpus_dist}

\subsection{Quality $\times$ Temporal Cross Distribution}
\label{app:quality_temporal}

Table~\ref{tab:cross_dist} presents the joint distribution of quality
tiers and temporal categories across the full preprocessed corpus
$\mathcal{D}_0$.

\begin{table}[t]
\centering
\small
\setlength{\tabcolsep}{0pt}
\renewcommand{\arraystretch}{1.06}

\begin{tabular*}{\columnwidth}{@{\extracolsep{\fill}} l l r r r r r @{}}
\toprule
\textbf{Qual.}
  & \textbf{Temp.}
  & \textbf{Docs}
  & \textbf{Doc\%}
  & \textbf{Tokens}
  & \textbf{Tok\%}
  & \textbf{Avg Tok.} \\
\midrule
\multirow{2}{*}{\textsc{HighQ}}
  & \textsc{TI}
  & 9.67\,M
  & 13.22\%
  & 17.19\,B
  & 13.06\%
  & 1{,}778 \\
  & \textsc{TR}
  & 5.46\,M
  & 7.47\%
  & 9.99\,B
  & 7.59\%
  & 1{,}828 \\
\midrule
\multirow{2}{*}{\textsc{MidQ}}
  & \textsc{TI}
  & 9.14\,M
  & 12.50\%
  & 12.22\,B
  & 9.29\%
  & 1{,}337 \\
  & \textsc{TR}
  & 9.15\,M
  & 12.52\%
  & 11.34\,B
  & 8.62\%
  & 1{,}239 \\
\midrule
\multirow{2}{*}{\textsc{LowQ}}
  & \textsc{TI}
  & 17.69\,M
  & 24.19\%
  & 62.46\,B
  & 47.46\%
  & 3{,}530 \\
  & \textsc{TR}
  & 22.02\,M
  & 30.11\%
  & 18.41\,B
  & 13.99\%
  & 836 \\
\midrule
\multicolumn{2}{@{}l}{Total}
  & 73.15\,M
  & 100\%
  & 131.63\,B
  & 100\%
  & 1{,}800 \\
\bottomrule
\end{tabular*}

\caption{Joint distribution of quality tiers and temporal categories in
\(\mathcal{D}_0\) 73.15\,M documents, 131.63\,B tokens.
\textsc{TI}: \textsc{Time\text{-}Irrel};
\textsc{TR}: \textsc{Time\text{-}Rel}.
Qual.: Quality; Temp.: Temporal; Docs: Documents;
Tok\%: Token percentage; Avg Tok.: Average tokens.}
\label{tab:cross_dist}
\end{table}

Notably, \textsc{LowQ} documents exhibit the highest average token
length (2{,}036 overall; Table~\ref{tab:quality_dist}), indicating
that document length alone is a poor proxy for quality.
Several patterns are noteworthy. Time-invariant \textsc{LowQ}
documents are disproportionately long (avg 3{,}530 tokens), likely
reflecting verbose, low-information-density web pages such as
repetitive template-generated content or promotional material.
Conversely, time-related \textsc{LowQ} content is short (avg 836
tokens), consistent with low-quality news snippets or brief social
media reposts. Within the \textsc{HighQ} tier, both temporal categories
show similar average lengths (${\sim}$1{,}800 tokens), suggesting that
quality assessment is largely independent of temporal characteristics
once the content passes the quality threshold.

\subsection{High-Quality Corpus Temporal Distribution}
\label{app:dhq_temporal}

Table~\ref{tab:dhq_temporal} summarizes the temporal breakdown of the
high-quality corpus $\mathcal{D}_{\text{dhq}}$.

\begin{table}[t]
\centering
\small
\setlength{\tabcolsep}{0pt}
\renewcommand{\arraystretch}{1.06}

\begin{tabular*}{\columnwidth}{@{\extracolsep{\fill}} l r r r r r @{}}
\toprule
\textbf{Temporal}
  & \textbf{Docs}
  & \textbf{Doc\%}
  & \textbf{Tokens}
  & \textbf{Tok\%}
  & \textbf{Avg Tok} \\
\midrule
\textsc{TI}
  & 9.67\,M
  & 63.91\%
  & 17.19\,B
  & 63.24\%
  & 1{,}778 \\
\textsc{TR}
  & 5.46\,M
  & 36.09\%
  & 9.99\,B
  & 36.76\%
  & 1{,}828 \\
\midrule
Total
  & 15.13\,M
  & 100\%
  & 27.18\,B
  & 100\%
  & 1{,}796 \\
\bottomrule
\end{tabular*}

\caption{Temporal distribution of \(\mathcal{D}_{\text{dhq}}\)
(high-quality subset). \textsc{TI}/\textsc{TR}: \textsc{Time\text{-}Irrel}
/ \textsc{Time\text{-}Rel}.}
\label{tab:dhq_temporal}
\end{table}

While the temporal dimension is nearly balanced in $\mathcal{D}_0$
(49.91\% \textsc{TI}; Table~\ref{tab:quality_dist}), it shifts to
63.91\% \textsc{TI} in $\mathcal{D}_{\text{dhq}}$, indicating that
time-invariant content accounts for a relatively higher proportion among
high-quality documents compared to the overall corpus. Within
$\mathcal{D}_{\text{dhq}}$, both temporal categories exhibit comparable
average lengths (1{,}778 tokens for \textsc{TI} vs.\ 1{,}828 tokens for
\textsc{TR}), indicating that the quality threshold is content-driven
rather than length-driven regardless of temporal category.

\subsection{Domain Distribution of High-Quality Documents}
\label{app:domain_dist}

Table~\ref{tab:domain_full} shows the distribution of document--concept
associations across the 12 top-level domains for
$\mathcal{D}_{\text{dhq}}$. Since each document is associated with
top-$M{=}3$ concept chains (\S\ref{sec:doc_concept}), the total count
(${\sim}$45.04\,M) is approximately
$3\times|\mathcal{D}_{\text{dhq}}|$.

\begin{table*}[t]
\centering\small
\begin{tabular}{@{}rl rr rr rr@{}}
\toprule
\textbf{Rk} & \textbf{Top-Level Domain}
  & \textbf{Assoc.}  & \textbf{\%}
  & \textbf{TI}      & \textbf{TI\%}
  & \textbf{TR}      & \textbf{TR\%} \\
\midrule
 1 & Technology \& Engineering & 15.67\,M & 34.80\% & 11.79\,M & 75.20\% & 3.89\,M & 24.80\% \\
 2 & Society                   &  6.89\,M & 15.29\% &  3.00\,M & 43.58\% & 3.89\,M & 56.42\% \\
 3 & Business \& Economy       &  5.88\,M & 13.05\% &  2.66\,M & 45.26\% & 3.22\,M & 54.74\% \\
 4 & Health \& Medicine        &  3.95\,M &  8.77\% &  3.08\,M & 78.06\% & 0.87\,M & 21.94\% \\
 5 & Culture \& Life           &  3.77\,M &  8.36\% &  2.86\,M & 76.05\% & 0.90\,M & 23.95\% \\
 6 & Arts \& Entertainment     &  3.23\,M &  7.17\% &  2.33\,M & 72.22\% & 0.90\,M & 27.78\% \\
 7 & Sports                    &  3.02\,M &  6.70\% &  0.99\,M & 32.96\% & 2.02\,M & 67.04\% \\
 8 & Natural Science           &  0.79\,M &  1.76\% &  0.67\,M & 85.07\% & 0.12\,M & 14.93\% \\
 9 & Philosophy \& Religion    &  0.62\,M &  1.38\% &  0.56\,M & 90.45\% & 59.3\,K &  9.55\% \\
10 & People \& Biography       &  0.48\,M &  1.07\% &  0.25\,M & 52.23\% & 0.23\,M & 47.77\% \\
11 & History                   &  0.48\,M &  1.07\% &  0.43\,M & 90.22\% & 46.9\,K &  9.78\% \\
12 & Travel \& Geography       &  0.26\,M &  0.58\% &  0.13\,M & 49.60\% & 0.13\,M & 50.40\% \\
\midrule
   & \textbf{Total}           & 45.04\,M & 100\%  & 28.78\,M & 63.88\% & 16.27\,M & 36.12\% \\
\bottomrule
\end{tabular}
\caption{Domain distribution of high-quality documents
($\mathcal{D}_{\text{dhq}}$) across 12 top-level domains.
\textbf{Assoc.}: document--concept association count.
\textbf{TI}/\textbf{TR}: \textsc{Time\text{-}Irrel} /
\textsc{Time\text{-}Rel} breakdown within each domain.}
\label{tab:domain_full}
\end{table*}

The temporal composition varies dramatically across domains.
Knowledge-intensive domains such as Natural Science (85.07\% TI),
History (90.22\%), and Philosophy \& Religion (90.45\%) are
overwhelmingly time-invariant, while event-driven domains such as
Sports (67.04\% TR) and Society (56.42\% TR) are predominantly
time-related. This validates the pipeline's temporal classification and
confirms that the OCG's domain organization captures meaningful
structural differences in web content.

\section{Continual Pre-Training Details}
\label{app:cpt_detail}

\subsection{CPT Validation Design}
\label{app:cpt_validation_design}

Concretely, the CPT experiments are designed to validate three aspects:
(a)~corpus quality, OCG domain organization, and domain-specific
retrieval accuracy through domain-specific performance improvements;
(b)~the OCG's inter-domain association modeling through the incremental
benefit of neighbor-domain data; and (c)~robustness of these benefits
across model families and scales.

All CPT experiments exclusively use \textsc{Time\text{-}Irrel} data, as
time-invariant content provides durable domain knowledge well-suited for
continual pre-training.

\paragraph{Scaling and Cross-Family Robustness.}
The four seed models spanning 0.5\,B to 8\,B parameters and two
model families (Qwen vs.\ Llama) enable analysis of scaling behavior
(0.5\,B to 8\,B) and cross-family robustness (Qwen vs.\ Llama),
verifying that the quality signal is intrinsic to the
\textsc{Cortex}-curated data rather than an artifact of specific
model-data interactions. Consistent improvement patterns across
scales and architectures further confirm the robustness of the
pipeline's quality filtering and the OCG's domain organization.

\subsection{OCG Neighbor Chains for Finance}
\label{app:finance_neighbors}

Table~\ref{tab:finance_neighbors} lists the top-10 OCG neighbor chains
for $c_{\text{fin}}$ (\texttt{Business \& Economy\,>\,Finance}), ranked
by $\operatorname{Assoc}(c_{\text{fin}},\cdot)$
(Eq.~\ref{eq:assoc}).

\begin{table}[t]
\centering\small
\begin{tabular}{@{}cl@{}}
\toprule
\textbf{Rank} & \textbf{Neighbor Chain} \\
\midrule
 1 & Business \& Economy\,>\,Corporate Mgmt. \\
 2 & Business \& Economy\,>\,Industrial Econ. \\
 3 & Society\,>\,Law \\
 4 & Business \& Economy\,>\,Consumer \& Retail \\
 5 & Business \& Economy\,>\,Real Estate \\
 6 & Tech.\,\& Eng.\,>\,Computer \& IT \\
 7 & People \& Biography\,>\,Entrepreneurs \\
 8 & Business \& Economy\,>\,Trade \\
 9 & Society\,>\,Career \& Employment \\
10 & Business \& Economy\,>\,Theoretical Econ. \\
\bottomrule
\end{tabular}
\caption{Top-10 OCG neighbor chains for $c_{\text{fin}}$, spanning 4
distinct top-level domains. This confirms that the OCG effectively
models latent associations across different domains.}
\label{tab:finance_neighbors}
\end{table}

\subsection{CPT Training Hyperparameters}
\label{app:cpt_hyperparams}

Table~\ref{tab:cpt_hparams} summarizes the CPT training configuration.

\begin{table}[t]
\centering\small
\begin{tabular}{@{}ll@{}}
\toprule
\textbf{Parameter} & \textbf{Value} \\
\midrule
Training mode      & Full-parameter \\
Learning rate      & $2{\times}10^{-6}$ \\
LR scheduler       & cosine \\
Warmup ratio       & 0.05 \\
Batch size (eff.)  & 128 \\
Precision          & bf16 \\
Max seq length     & 4{,}096 \\
Inference engine   & vLLM \\
Eval (CFBenchmark) & Official toolkit \\
\bottomrule
\end{tabular}
\caption{CPT training configuration (shared across all seed models
and data configurations).}
\label{tab:cpt_hparams}
\end{table}

\subsection{Domain Capability Benchmark}
\label{app:cfb_benchmark}

The domain capability benchmark referred to as ``CFBenchmark'' in the
main text corresponds to
CFBenchmark-OpenFinData~\cite{CFBenchmark}, which evaluates financial
domain capabilities across six dimensions. The column abbreviations
used in Table~\ref{tab:cpt_results} are:
\textbf{Know.}:~Knowledge,
\textbf{Calc.}:~Calculation,
\textbf{Expl.}:~Explanation,
\textbf{Ident.}:~Identification,
\textbf{Anal.}:~Analysis, and
\textbf{Compl.}:~Compliance.
Each dimension targets a distinct aspect of financial language
understanding and reasoning. All evaluations are conducted using the
official CFBenchmark evaluation toolkit.

\section{\textsc{CortexBench} Details}
\label{app:synth_detail}

\subsection{Benchmark Design Rationale}
\label{app:bench_rationale}

\textsc{CortexBench} is a cross-domain search-and-reasoning QA
dataset in which each instance requires integrating evidence from two
documents belonging to different yet OCG-associated domains. Each
instance is paired with a dedicated candidate document pool
($|\mathcal{S}_{\text{cand}}|{=}6{,}000$) and a deterministic hybrid
search interface. Unlike settings that rely on live web search APIs,
where dynamic and opaque retrieval environments hinder fair
comparisons and reproducibility, and make it difficult to isolate
specific system components or attribute the sources of model
failures~\citep{BrowseComp-Plus}, this design provides a stable,
transparent interaction environment that enables reproducible
evaluation and fine-grained attribution of model weaknesses to
specific capabilities such as query formulation, evidence retrieval,
or cross-domain reasoning.
Concretely, this enables attribution of weaknesses to specific
capabilities such as query formulation, evidence retrieval, or
cross-domain reasoning.

The candidate pool comprises four equal-sized partitions: documents
from the two evidence chains $c_1$ and $c_2$ (including the gold
evidence $d_A$ and $d_B$), and documents from two randomly selected
chains $c_{r_1},c_{r_2}\in\mathcal{C}\setminus\{c_1,c_2\}$. This
design creates a controlled yet challenging retrieval setting: the
gold evidence documents are embedded among same-domain distractors
that share topical similarity, while cross-domain and random
distractors test the model's retrieval precision and noise robustness.

\subsection{Synthesis Pipeline}
\label{app:synth_pipeline}

Algorithm~\ref{alg:synthesis} formalizes the OCG-driven evidence pair
selection, QA synthesis, and per-QA-instance candidate set construction
described in \S\ref{sec:bench_construct}. Let
$\mathcal{M}_{\text{ext}}$ denote the LLM-based entity extraction
model, $\mathcal{E}_{\text{m3}}$ the BGE-M3 embedding
model~\citep{bge-m3}, and $\mathcal{M}_{\text{syn}}$ the QA
synthesis LLM. Both $\mathcal{M}_{\text{ext}}$ and
$\mathcal{M}_{\text{syn}}$ are implemented using
GPT-5.4.

\begin{algorithm}[t]
\caption{OCG-Guided QA Synthesis and Candidate Set Construction}
\label{alg:synthesis}
\small
\begin{algorithmic}[1]
\Require Keyword vocabulary $\mathcal{K}$; concept chains $\mathcal{C}$;
  inverted index $\{\mathcal{D}_c\}$;
  $\operatorname{Assoc}(\cdot,\cdot)$;
  entity extractor $\mathcal{M}_{\text{ext}}$;
  embedding model $\mathcal{E}_{\text{m3}}$;
  QA synthesis LLM $\mathcal{M}_{\text{syn}}$;
  pool size $s{=}6{,}000$
\Ensure QA set $\{(q, a, d_A, d_B,
  \mathcal{S}_{\text{cand}})\}$
\For{each selected keyword $k \in \mathcal{K}$}
  \State $\mathcal{C}_k \gets
    \{c\in\mathcal{C}\mid\operatorname{F}(c,k)>0\}$
  \State Select $(c_1,c_2)$ from $\mathcal{C}_k$
    (high-assoc.\ or low-assoc.)
  \State $\mathcal{D}_{c_i}^{\text{top}} \gets
    \operatorname{top\text{-}}N_{d\in\mathcal{D}_{c_i}}
    \alpha(d,c_i)$, $i\in\{1,2\}$
  \For{each $d_A \in
    \operatorname{sample}(\mathcal{D}_{c_1}^{\text{top}}, 5)$}
    \State $\mathcal{E}_A \gets
      \mathcal{M}_{\text{ext}}(d_A)$
      \Comment{Extract entities \& attributes}
    \If{$\max_{e\in\mathcal{E}_A}\operatorname{score}(e)
      < \tau_{\text{ent}}$}
      \State \textbf{continue}
      \Comment{Skip low-salience $d_A$}
    \EndIf
    \State $\mathcal{B} \gets
      \operatorname{top\text{-}}n_{d\in
      \mathcal{D}_{c_2}^{\text{top}}}
      \cos\bigl(\mathcal{E}_{\text{m3}}(\mathcal{E}_A),\;
      \mathcal{E}_{\text{m3}}(d)\bigr)$
    \For{each $d_B \in \mathcal{B}$}
      \State $(q,a) \gets
        \mathcal{M}_{\text{syn}}(d_A, d_B, \text{type})$,
        $\text{type}\in\{\text{bridge},\text{comparison}\}$
      \If{quality filter passes}
        \State \textbf{Build candidate pool:}
        \State \quad $\mathcal{S}_{c_1} \gets
          \{d_A\}\cup
          \operatorname{sample}\bigl(\mathcal{D}_{c_1}
          \!\setminus\!\{d_A\},\;s/4{-}1\bigr)$
        \State \quad $\mathcal{S}_{c_2} \gets
          \{d_B\}\cup
          \operatorname{sample}\bigl(\mathcal{D}_{c_2}
          \!\setminus\!\{d_B\},\;s/4{-}1\bigr)$
        \State \quad Select random $c_{r_1},c_{r_2}
          \in\mathcal{C}\setminus\{c_1,c_2\}$
        \State \quad $\mathcal{S}_{\text{cand}} \gets
          \mathcal{S}_{c_1}\cup\mathcal{S}_{c_2}\cup
          \operatorname{sample}(\mathcal{D}_{c_{r_1}},s/4)\cup
          \operatorname{sample}(\mathcal{D}_{c_{r_2}},s/4)$
        \State \textbf{emit}
          $(q,a,d_A,d_B,\mathcal{S}_{\text{cand}})$;
          \textbf{break}
        \Comment{$d_A$ yields $\le$1 QA}
      \EndIf
    \EndFor
  \EndFor
\EndFor
\end{algorithmic}
\end{algorithm}

In line~10, $n{=}5$ for bridge and $n{=}10$ for comparison questions.
Each keyword $k$ contributes at most 5 candidate $d_A$ documents, and
each $d_A$ generates at most one QA pair (line~19 breaks after the
first successful synthesis). The candidate pool size per chain is
$N{=}40$. Within the bridge and comparison categories, a small
proportion (${\sim}$10\%) of QA pairs have their premises deliberately
modified to contain factual errors; for these instances, a model is
judged correct only if it identifies the erroneous premise, further
evaluating the model's reasoning capability.

\paragraph{QA Types and Quality Filtering.}
Two QA types are synthesized: \emph{bridge} questions require chaining
facts across $d_A$ and $d_B$, while \emph{comparison} questions
require contrasting entities or attributes across the two documents.
A quality filter rejects QA pairs whose answers cannot be traced back
to the evidence documents, ensuring that every retained question is
answerable given the source material.

\subsection{Benchmark Statistics}
\label{app:bench_stats}

The 917 QA pairs are organized along three dimensions:
(i)~\emph{question type}: bridge and comparison
(Appendix~\ref{app:synth_pipeline});
(ii)~\emph{temporal source}: whether evidence documents are drawn
from the \textsc{Time\text{-}Irrel} or \textsc{Time\text{-}Rel}
subset of $\mathcal{D}_{\text{refined}}$; and
(iii)~\emph{chain-pair association level}: whether the two evidence
chains have high or low $\operatorname{Assoc}$ scores
(\S\ref{sec:graph_construct}).
Table~\ref{tab:bench_stats} summarizes the composition.

\begin{table}[t]
\centering\small
\begin{tabular}{@{}ll r@{}}
\toprule
\textbf{Dimension} & \textbf{Category} & \textbf{Count} \\
\midrule
\multirow{2}{*}{QA Type}
  & Bridge         & 320 \\
  & Comparison     & 597 \\
\midrule
\multirow{2}{*}{Temporal Source}
  & \textsc{TI} source & 544 \\
  & \textsc{TR} source & 373 \\
\midrule
\multirow{2}{*}{Association Level}
  & High-association pairs & 512 \\
  & Low-association pairs  & 405 \\
\midrule
\multicolumn{2}{l}{\textbf{Total}} & \textbf{917} \\
\bottomrule
\end{tabular}
\caption{\textsc{CortexBench} composition across QA type, temporal
source, and chain-pair association level.}
\label{tab:bench_stats}
\end{table}

\subsection{Search Interface and Evaluation Protocol}
\label{app:search_impl}

\paragraph{Evaluated Models.}
We evaluate eight frontier LLMs spanning four providers:
GPT-4o and GPT-5.4 (OpenAI);
Claude Sonnet~4.6 and Claude Opus~4.6 (Anthropic);
Gemini~3~Pro and Gemini~3.1~Pro (Google);
DeepSeek-V3.2 and DeepSeek-V4-Flash (DeepSeek).
Accuracy is evaluated by an LLM-based judge that assesses semantic
equivalence between the predicted and gold answers
(Appendix~\ref{app:judge}). Evidence Hit@$k$ measures whether at
least one gold evidence document appears among the top-$k$ retrieved
chunks (Appendix~\ref{app:bench_hit}).

\paragraph{Candidate Pool Construction.}
For each QA instance $(q,a,d_A,d_B,\mathcal{S}_{\text{cand}})$, the
candidate pool $\mathcal{S}_{\text{cand}}$
($|\mathcal{S}_{\text{cand}}|{=}6{,}000$) is constructed during
synthesis (Algorithm~\ref{alg:synthesis}, lines~14--18). Each QA pair
has its own dedicated candidate pool.

\paragraph{Document Chunking.}
Because the embedding model used for retrieval
(BGE-Large-zh~\citep{BGE-large-zh}) has a maximum input length of 512
tokens, each document $d\in\mathcal{S}_{\text{cand}}$ exceeding 512
tokens is split into non-overlapping chunks of at most 512 tokens. All
subsequent retrieval operates at the chunk level.

\paragraph{Hybrid Scoring.}
For a query $(\mathbf{q}_{\text{kw}}, q_{\text{sum}})$ and each chunk
$d_j$, two scores are computed:
$s_{\text{sem}}(d_j) = \cos\bigl(
  \mathcal{E}_{\text{bge}}(q_{\text{sum}}),\;
  \mathcal{E}_{\text{bge}}(d_j)\bigr)$
using BGE-Large-zh (1{,}024-dim, 512 max tokens), and
$s_{\text{lex}}(d_j) = \operatorname{BM25}(\mathbf{q}_{\text{kw}},
d_j)$ with default parameters ($k_1{=}1.5$, $b{=}0.75$).
Each score is min-max normalized within the pool, and the final score is
$s(d_j) = 0.5\,\hat{s}_{\text{sem}}(d_j) +
0.5\,\hat{s}_{\text{lex}}(d_j)$.
The function returns the top-$m$ chunks ranked by $s$.

\paragraph{Caching.}
Chunk embeddings for each candidate pool are pre-computed and cached.
BM25 inverted indices are built per pool and reused across search
rounds within the same QA instance.

\paragraph{Evaluation Protocol.}
Algorithm~\ref{alg:eval} formalizes the evaluation procedure for
\textsc{CortexBench}. The evaluated model $\mathcal{M}$ receives only
the question $q$ and must autonomously extract a keyword list
$\mathbf{q}_{\text{kw}}$, compose a natural-language summary
$q_{\text{sum}}$, and specify the number of chunks to retrieve $m$
($m{\le}5$) before each search call. In the \textsc{Multi-Search}
setting, the model may issue up to 3 rounds of search but may also
choose to generate an answer at any earlier round based on the
accumulated evidence.

\begin{algorithm}[t]
\caption{\textsc{CortexBench} Evaluation Protocol}
\label{alg:eval}
\small
\begin{algorithmic}[1]
\Require QA instance $(q,a,d_A,d_B,\mathcal{S}_{\text{cand}})$;
  evaluated model $\mathcal{M}$; setting
  $\in\{\textsc{Closed},\textsc{1-Srch},\textsc{M-Srch},
  \textsc{Oracle}\}$
\Ensure Correctness $\in\{0,1\}$
\If{setting $=$ \textsc{Closed}}
  \State $\hat{a} \gets \mathcal{M}(q)$
  \Comment{Parametric knowledge only}
\ElsIf{setting $=$ \textsc{Oracle}}
  \State $\hat{a} \gets \mathcal{M}(q,\,d_A,\,d_B)$
  \Comment{Gold evidence provided}
\Else
  \State $(\mathbf{q}_{\text{kw}},\,q_{\text{sum}},\,m)
    \gets \mathcal{M}(q)$,
    \quad $m \le 5$
  \Comment{Model sees only $q$}
  \State $\mathcal{R} \gets
    \operatorname{Search}(\mathbf{q}_{\text{kw}},
    q_{\text{sum}},m;\;\mathcal{S}_{\text{cand}})$
  \If{setting $=$ \textsc{1-Srch}}
    \State $\hat{a} \gets \mathcal{M}(q,\,\mathcal{R})$
  \Else
    \Comment{\textsc{M-Srch}: up to 3 rounds via ReAct}
    \State $\mathcal{R}_{\text{all}} \gets \mathcal{R}$;\;
      $r \gets 1$
    \While{$r < 3$}
      \If{$\mathcal{M}$ decides to answer}
        \textbf{break}
      \EndIf
      \State $(\mathbf{q}_{\text{kw}}',q_{\text{sum}}',m')
        \gets \mathcal{M}(q,\mathcal{R}_{\text{all}})$,
        \;$m'\le 5$
      \State $\mathcal{R}' \gets
        \operatorname{Search}(\mathbf{q}_{\text{kw}}',
        q_{\text{sum}}',m';\;\mathcal{S}_{\text{cand}})$
      \State $\mathcal{R}_{\text{all}} \gets
        \mathcal{R}_{\text{all}}\cup\mathcal{R}'$;\;
        $r \gets r+1$
    \EndWhile
    \State $\hat{a} \gets
      \mathcal{M}(q,\,\mathcal{R}_{\text{all}})$
  \EndIf
\EndIf
\State \textbf{return}
  $\operatorname{Judge}(q,\,a,\,\hat{a})$
\end{algorithmic}
\end{algorithm}

\subsection{LLM-as-Judge Evaluation}
\label{app:judge}

We use GPT-5.4 as the evaluator. Given the
original question $q$, the gold answer $a$, and the predicted answer
$\hat{a}$, the evaluator assesses semantic equivalence rather than
exact string matching, outputting a binary judgment
(correct/incorrect). The evaluator prompt template is specified in our
released codebase.

\subsection{Evidence Hit Rate Analysis}
\label{app:bench_hit}

Table~\ref{tab:evidence_hit} reports Evidence Hit@$k$ for each model.
Evidence Hit@$k$ is the fraction of instances where at least one gold
evidence document ($d_A$ or $d_B$) appears among the top-$k$ retrieved
chunks.

\begin{table}[t]
\centering
\small
\setlength{\tabcolsep}{4pt}
\renewcommand{\arraystretch}{1.08}

\begin{tabularx}{\columnwidth}{@{}l *{4}{Y}@{}}
\toprule
& \multicolumn{2}{c}{\textbf{Bridge}}
& \multicolumn{2}{c}{\textbf{Comparison}} \\
\cmidrule(lr){2-3}\cmidrule(lr){4-5}
\textbf{Model}
  & Hit@5 & Hit@15
  & Hit@5 & Hit@15 \\
\midrule
DeepSeek-V3.2            & 22.7 & 31.4 & 28.6 & 45.2 \\

DeepSeek-V4-Flash        & 23.1 & 31.6 & 30.5 & 45.5 \\
Gemini 3 Pro       & 22.5 & \textbf{55.0} & 24.6 & \textbf{50.9} \\
Gemini 3.1 Pro     & \textbf{38.4} & 14.1 & \textbf{46.0} & 12.6 \\
Claude Sonnet~4.6    & 20.3 & \underline{41.1} & 29.8 & 35.7 \\
Claude Opus~4.6     & 20.6 & 40.8 & 29.1 & \underline{49.2} \\
GPT-4o             & \underline{33.4} & 27.9 & \underline{36.6} & 33.9 \\

GPT-5.4            & 24.4 & 28.6 & 29.6 & 39.5 \\
\bottomrule
\end{tabularx}

\caption{Evidence Hit rates on \textsc{CortexBench}.
\textbf{Hit@5}: 1-Search top-5.
\textbf{Hit@15}: Multi-Search cumulative (\(\le\)3\(\times\)5).
}
\label{tab:evidence_hit}
\end{table}

\subsection{Results by Association Level and Temporal Source}
\label{app:bench_breakdown}

Table~\ref{tab:bench_combined_breakdown} provides accuracy breakdowns
by chain-pair association level and temporal source data.

\begin{table*}[t]
\centering
\small
\setlength{\tabcolsep}{2pt}
\renewcommand{\arraystretch}{1.08}

\begin{tabularx}{\textwidth}{@{}l *{16}{Y}@{}}
\toprule
& \multicolumn{8}{c}{\textbf{By Chain-Pair Association Level}}
& \multicolumn{8}{c}{\textbf{By Temporal Source}} \\
\cmidrule(lr){2-9}\cmidrule(lr){10-17}
& \multicolumn{4}{c}{\textbf{High-Association}}
& \multicolumn{4}{c}{\textbf{Low-Association}}
& \multicolumn{4}{c}{\textbf{\textsc{Time\text{-}Irrel} Source}}
& \multicolumn{4}{c}{\textbf{\textsc{Time\text{-}Rel} Source}} \\
\cmidrule(lr){2-5}\cmidrule(lr){6-9}
\cmidrule(lr){10-13}\cmidrule(lr){14-17}
\textbf{Model}
  & \textsc{Cls.} & \textsc{1-S} & \textsc{M-S} & \textsc{Orc.}
  & \textsc{Cls.} & \textsc{1-S} & \textsc{M-S} & \textsc{Orc.}
  & \textsc{Cls.} & \textsc{1-S} & \textsc{M-S} & \textsc{Orc.}
  & \textsc{Cls.} & \textsc{1-S} & \textsc{M-S} & \textsc{Orc.} \\
\midrule
DeepSeek-V3.2
  & 24.0 & 34.2 & 48.4 & \textbf{96.0}
  & 26.9 & 31.9 & 49.6 & 92.4
  & 27.2 & 33.8 & 48.5 & \underline{93.6}
  & 22.5 & 32.2 & 49.6 & \textbf{95.4} \\

DeepSeek-V4-Flash
  & 25.8 & 35.5 & 50.0 & 91.5
  & 28.4 & 34.6 & 51.9 & 92.3
  & 27.9 & 38.4 & 52.2 & 91.4
  & 25.5 & 30.3 & 48.8 & 92.5 \\
Gemini 3 Pro
  & 29.3 & 34.6 & \underline{57.4} & 72.7
  & 29.6 & 32.6 & 52.8 & 78.2
  & 32.9 & 36.4 & 54.4 & 72.3
  & 24.4 & 29.8 & \underline{56.7} & 79.2 \\
Gemini 3.1 Pro
  & 26.8 & \textbf{61.4} & 44.7 & 68.2
  & 24.2 & 35.1 & 44.4 & 75.8
  & 27.9 & \textbf{58.3} & 48.2 & 72.5
  & 22.3 & \textbf{59.4} & 39.4 & 70.3 \\
Claude Sonnet 4.6
  & 27.7 & 46.5 & 46.7 & 82.8
  & 29.1 & \underline{44.4} & 48.1 & 84.2
  & 30.0 & 46.7 & 49.3 & 83.2
  & 26.0 & 44.0 & 44.5 & 83.7 \\
Claude Opus 4.6
  & \underline{32.0} & 46.3 & \textbf{60.2} & 94.7
  & \underline{34.1} & 42.2 & \textbf{57.5} & \underline{94.0}
  & \underline{34.4} & 46.7 & \textbf{60.5} & \textbf{95.2}
  & \underline{30.8} & 41.3 & \textbf{56.8} & 93.3 \\
GPT-4o
  & 21.3 & 23.8 & 30.8 & 73.0
  & 25.9 & 25.2 & 30.6 & 76.9
  & 24.6 & 24.8 & 33.3 & 75.3
  & 21.4 & 23.9 & 26.9 & 74.0 \\

GPT-5.4
  & \textbf{36.5} & \underline{53.7} & 55.7 & \underline{95.7}
  & \textbf{38.8} & \textbf{49.6} & \underline{53.8} & \textbf{94.2}
  & \textbf{40.8} & \underline{54.4} & \underline{55.5} & \textbf{95.2}
  & \textbf{32.7} & \underline{48.3} & 53.9 & \underline{94.9} \\
\bottomrule
\end{tabularx}

\caption{Accuracy breakdowns on \textsc{CortexBench}
under chain-pair association level and temporal source settings
(bridge and comparison combined).
\textsc{Cls.}: Closed-book.
\textsc{1-S}: 1-Search.
\textsc{M-S}: Multi-Search.
\textsc{Orc.}: Oracle.}
\label{tab:bench_combined_breakdown}
\end{table*}

\subsection{Detailed Results Analysis}
\label{app:detailed_analysis}

We provide a systematic analysis of \textsc{CortexBench} results
across Table~\ref{tab:benchmark_results},
Table~\ref{tab:evidence_hit}, and
Table~\ref{tab:bench_combined_breakdown}.

\paragraph{Cross-Setting Performance Progression.}
Across all eight models, the four evaluation settings reveal a
consistent progression pattern. In the \textsc{Closed-book} setting,
bridge questions (avg 33.0\%) prove slightly more accessible than
comparison questions (avg 26.4\%), likely because bridge questions
sometimes permit partial answers from parametric knowledge about
individual entities. The \textsc{1-Search} setting produces highly
variable improvements: Gemini~3.1~Pro achieves the largest gain on
bridge (+37.6 points to 65.7\%), while GPT-4o gains only 0.3 points,
suggesting that single-round query formulation quality varies
dramatically across model families. The \textsc{Multi-Search} setting
benefits models with strong iterative planning (Gemini~3~Pro gains
24.9 points from 1-Search to reach 63.0\% on bridge), but notably
degrades Gemini~3.1~Pro by 13.2 points on bridge and 13.7 on
comparison, indicating that additional search rounds can introduce
noise when the initial retrieval was already highly effective.

\paragraph{Evidence Retrieval Analysis.}
Table~\ref{tab:evidence_hit} reveals a strong correspondence between
evidence hit rates and accuracy. Gemini~3.1~Pro achieves the highest
Hit@5 (38.4\% bridge, 46.0\% comparison) but the lowest Hit@15
(14.1\%, 12.6\%), explaining its high 1-Search accuracy but
Multi-Search degradation: its initial queries are highly effective,
but subsequent queries fail to locate additional evidence and may
introduce distracting content. Conversely, Gemini~3~Pro shows the
opposite pattern with low Hit@5 (22.5\%, 24.6\%) but the highest
Hit@15 (55.0\%, 50.9\%), explaining its large Multi-Search gains.
GPT-4o achieves relatively high Hit@5 (33.4\%, 36.6\%) but
translates this poorly into accuracy gains (+0.3 on bridge), suggesting
a reasoning bottleneck rather than a retrieval bottleneck.

\paragraph{Association Level Effects.}
Table~\ref{tab:bench_combined_breakdown} shows 
the Oracle gap differs: high-association
pairs generally yield higher Oracle accuracy (e.g., DeepSeek-V3.2:
96.0\% vs.\ 92.4\%), suggesting that high-association evidence pairs
produce more coherent cross-domain reasoning chains. The Multi-Search
advantage is generally preserved across both levels, confirming the
robustness of the benchmark design.

\paragraph{Temporal Source Effects.}
\textsc{Time\text{-}Irrel} source instances yield slightly higher
closed-book accuracy than \textsc{Time\text{-}Rel} instances across
most models (e.g., GPT-5.4: 40.8\% vs.\ 32.7\%), consistent with
time-invariant content being more likely to overlap with pre-training
data. The search-augmented settings largely close this gap (e.g.,
Claude~Opus~4.6 Multi-Search: 60.5\% TI vs.\ 56.8\% TR),
demonstrating that the retrieval mechanism effectively compensates
for the parametric knowledge deficit on time-sensitive content.
Oracle accuracy shows minimal temporal variation, confirming that
both temporal categories produce questions of comparable reasoning
difficulty once evidence is provided.

\end{document}